\documentclass[11pt, a4paper, logo, copyright]{googledeepmind}
\usepackage[authoryear, sort&compress, round]{natbib}
\usepackage{multicol}
\usepackage{multirow}
\usepackage{subcaption}
\usepackage{tablefootnote}
\usepackage{algorithmic}
\usepackage{listings}
\usepackage{pifont}
\usepackage{amsmath, amssymb}
\usepackage{mdframed} % For framing quotes and logs
\usepackage{enumitem} % For better list formatting

\usepackage{hyperref}
\usepackage{tabularx}
\usepackage{booktabs}
\usepackage[table]{xcolor}
\usepackage{ragged2e}
\usepackage{float}

\theoremstyle{plain}
\newtheorem{theorem}{Theorem}
\newtheorem{lemma}{Lemma}

\theoremstyle{definition}

\newtheorem{remark}{Remark}[section]
\newtheorem{definition}[theorem]{Definition}

\definecolor{humanbubble}{RGB}{235, 245, 255}
\definecolor{aibubble}{RGB}{245, 245, 245}
\definecolor{loggray}{RGB}{230, 230, 230}
\definecolor{indexbg}{RGB}{248, 248, 248}

\usepackage[most]{tcolorbox}

\newtcolorbox{interactionlog}[2][]{
  enhanced,
  arc=0pt, outer arc=0pt,
  colback=white, colframe=black!60,
  boxrule=0.8pt,
  fonttitle=\bfseries\sffamily, coltitle=black, colbacktitle=loggray,
  title={Human-AI Interaction Card \if\relax\detokenize{#1}\relax\else for #1\fi},
  halign title=center, attach title to upper,
  after title={\vspace{4pt}\hrule\vspace{10pt}},
  lower separated=true,
  segmentation style={solid, black!60, line width=0.8pt},
  colbacklower=indexbg,
  after upper={\par\vfill
    \begin{tcolorbox}[
      enhanced, colback=indexbg, colframe=white, boxrule=0pt,
      top=0pt, bottom=0pt, fontupper=\footnotesize\sffamily,
      title=Raw prompts and outputs, coltitle=black!70, attach title to upper,
      after title={:\enskip}, sharp corners
    ]
    #2 
    \end{tcolorbox}
  }
}

\newcommand{\human}[1]{%
  \noindent\begin{flushright}
    \begin{minipage}[c]{0.70\textwidth}
      \begin{tcolorbox}[
        enhanced,
        colback=humanbubble, colframe=black!15,
        arc=6pt, sharp corners=southeast, boxrule=0.5pt,
        left=6pt, right=6pt, top=4pt, bottom=4pt, boxsep=0pt
      ]\small #1\end{tcolorbox}
    \end{minipage}%
    \hspace{8pt}
    \begin{minipage}[c]{40pt}
      \footnotesize\sffamily\textbf{Human}
    \end{minipage}
  \end{flushright}
  \vspace{-12pt}
}

\newcommand{\ai}[2]{%
  \noindent\begin{flushleft}
    \begin{minipage}[c]{55pt}
      \footnotesize\sffamily\textbf{#1}
    \end{minipage}%
    \hspace{2pt}
    \begin{minipage}[c]{0.65\textwidth}
      \begin{tcolorbox}[
        enhanced,
        colback=aibubble, colframe=black!15,
        arc=6pt, sharp corners=southwest, boxrule=0.5pt,
        left=6pt, right=6pt, top=4pt, bottom=4pt, boxsep=0pt
      ]\small #2\end{tcolorbox}
    \end{minipage}
  \end{flushleft}
  \vspace{-12pt}
}

\definecolor{boxblue}{RGB}{0, 0, 150}
\definecolor{boxback}{RGB}{245, 245, 255}

\newtcolorbox{problem}[1]{%
    colback=boxback,
    colframe=boxblue,
    fonttitle=\bfseries\large,
    title={#1},
    sharp corners,
    enhanced,
    attach boxed title to top left={yshift=-2mm, xshift=2mm},
    boxed title style={colframe=boxblue, colback=boxblue},
    before skip=15pt plus 2pt,
    after skip=15pt plus 2pt,
    top=10pt, bottom=10pt, left=10pt, right=10pt
}

\newtcolorbox{solution}[1]{%
    colback=white,
    colframe=boxblue,
    fonttitle=\bfseries\large,
    title={#1},
    sharp corners,
    enhanced jigsaw, % Better frame handling for page breaks than just 'enhanced'
    breakable,       % <--- Allows the box to split across pages
    attach boxed title to top left={yshift=-2mm, xshift=2mm},
    boxed title style={colframe=boxblue, colback=boxblue},
    before skip=15pt plus 2pt,
    after skip=15pt plus 2pt,
    top=10pt, bottom=10pt, left=10pt, right=10pt
}

\definecolor{tablegray}{HTML}{EFEFEF}

\newcommand{\aletheia}{\emph{Aletheia}}
\newcommand{\fp}{\emph{FirstProof}}
\newcommand{\agenta}{\emph{Aletheia} A}
\newcommand{\agentb}{\emph{Aletheia} B}

\DeclareMathOperator{\vecop}{vec}

\DeclareMathAlphabet{\catsymbfont}{U}{rsfs}{m}{n}

\newcommand{\bR}{\mathbb{R}}

\newcommand{\aO}{{\catsymbfont{O}}}

\title{\aletheia{} tackles \fp{} autonomously}
\author{
 Tony Feng\textsuperscript{*},
 Junehyuk Jung,
 Sang-hyun Kim,
 Carlo Pagano,
 Sergei Gukov,
 Cheng-Chiang Tsai,
 David P. Woodruff,
 Adel Javanmard, 
 Aryan Mokhtari, 
 Dawsen Hwang, 
 Yuri Chervonyi,
 Jonathan N. Lee, 
 Garrett Bingham,
 Trieu H. Trinh,
 Vahab Mirrokni, 
 Quoc V. Le,
 Thang Luong\textsuperscript{*}
\\
 \textsuperscript{*}Project leads.
 Work conducted under Google DeepMind. 
}
\correspondingauthor{fengtony@google.com, thangluong@google.com. \newline External affiliations: UC Berkeley (Tony Feng), Brown University (Junehyuk Jung), Korea Institute for Advanced Study (Sang-hyun Kim), Concordia University (Carlo Pagano), Caltech (Sergei Gukov), Academia Sinica (Cheng-Chiang Tsai), CMU (David Woodruff), USC (Adel Javanmard), UT Austin (Aryan Mokhtari).}
 
\begin{abstract}
We report the performance of \aletheia{} \citep{feng2026autonomousmathematicsresearch}, a mathematics research agent powered by Gemini 3 Deep Think, on the inaugural \fp{} challenge.
Within the allowed timeframe of the challenge, \aletheia{} autonomously solved 6 problems (2, 5, 7, 8, 9, 10) out of 10 according to 
majority expert assessments; we note that experts were not unanimous on Problem 8 (only). For full transparency, we explain our interpretation of \fp{} and disclose details about our experiments as well as our evaluation.
Raw prompts and outputs are available at \url{https://github.com/google-deepmind/superhuman/tree/main/aletheia}.
\end{abstract}

\begin{document}

\maketitle

\section{Introduction}
\fp{} \citep{abouzaid2026proof} is a collection of ten research-level math problems that arose naturally in the work of professional mathematicians, and was proposed as an assessment of current AI capabilities. These problems are described by the \fp{} authors as being ``Lemmas'', meaning intermediate technical statements rather than open problems of interest for their own sake\footnote{At least one (Problem 7) had also been previously described as an open problem of interest \cite{Wein23}.}. % All of these problems had already been solved by mathematicians, but the solutions did not appear online\footnote{A sketch of the solution to Problem 1 could be found online.}. 
The problems were released on February 5, 2026 and given a deadline of 11:59pm PST on February 13, 2026, at which point official (human-written) solutions were published.

This report documents the performance of \aletheia{} \citep{feng2026autonomousmathematicsresearch}, a mathematics research agent powered by Gemini 3 Deep Think~\citep{gemini3deepthink2026},  on \fp{}. 
 The results of \aletheia{} on a best-of-2 run per problem are displayed in Table \ref{table:intro-results}.

\begin{table}[H]
\centering
\renewcommand{\arraystretch}{1.2} 
\begin{tabular}{|l|c|c|}
\hline
    & \shortstack{\aletheia{} \\ (best of 2)}  & \shortstack{Expert Evaluation \\ (correct/total)}\\ \hline
P1  & No Output  & \\ \hline
P2  & \color{green} Correct & 4/4 \\ \hline
P3  & No Output &  \\ \hline
P4  & No Output & \\ \hline
P5  & \color{green} Correct & 4/4 \\ \hline
P6  & No Output & \\ \hline
P7  & \color{green} Correct &  3/3\\ \hline
P8  & \color{green!70} Correct? & 5/7\\ \hline
P9  & \color{green} Correct & 4/4  \\  \hline
P10 & \color{green} Correct & 2/2 \\   \hline
\end{tabular}
\caption{Summary of \aletheia{}'s performance on \fp{}. The Expert Evaluation column displays the number of experts who rated the solution as being Correct, out of the total number of experts consulted. Only the assessment on P8 was not unanimous.}
\label{table:intro-results}
\end{table}

 We emphasize that \emph{this is a limited study conducted by the team behind the \aletheia{} agent, with help on evaluations from other experts within Google; it is not representative of Google's collective efforts on \fp{}}.

Continuing our practice of transparency in AI for mathematical and scientific discovery \citep{discovery26} and the concept of Human-AI Interaction (HAI) card introduced in \citep{feng2026autonomousmathematicsresearch}, we provide below the HAI card for how we have obtained solutions to \fp{}.
\begin{interactionlog}[]{https://github.com/google-deepmind/superhuman/tree/main/aletheia}

    \human{Prompt (problems copy-pasted verbatim from \fp{} .tex file)}

    \ai{\aletheia}{Response}

    \human{Verification and extraction prompt (\S \ref{sec:extraction-prompt}) }

    \ai{Gemini\,3 Deep Think}{Response}

\end{interactionlog}

\section{Interpretation of the challenge}

Since \fp{} was construed as an experimental trial without clearly defined rules, we first discuss our interpretation of the challenge. The \fp{} authors write in the FAQ of \url{1stproof.org}, 

\begin{quote}
\emph{What constitutes a solution?}

We consider that an AI model has answered one of our questions if it can produce in an autonomous way a proof that conforms to the levels of rigor and scholarship prevailing in the mathematics literature. In particular, the AI should not rely on human input for any mathematical idea or content, or to help it isolate the core of the problem. Citations should include precise statement numbers and should either be to articles published in peer-reviewed journals or to arXiv preprints.
\end{quote}

and in the paper \citep{abouzaid2026proof}, 

\begin{quote}
...it is not yet clear where AI systems stand at solving research-level math questions
on their own, without an expert in the loop.
\end{quote}

\textbf{Autonomy.} Even with these guidelines, we felt some ambiguity about what constitutes an ``autonomous solution''.
For example: if an AI produces a proof, and a human reviewer asks for clarification on a technical point, and the AI then elaborates to make the argument more rigorous, is the result considered autonomous? This sort of interaction happens all the time in human peer review. We think the answer could be ``yes'', at least if the full transcript of the interaction is provided and observers agree that the human input does not contain mathematical ideas or content. On the other hand, for research problems at the caliber of \fp{}, expertise is already required to identify possible weak points to ask about, so such an interaction cannot occur without an expert in the loop. 

Another question was whether human expertise can be used to select the best solution out of a number of attempts. Under our reading of the rules, this does not seem to be disallowed. But it offers a potentially huge performance advantage, which seems orthogonal to the evaluation of AI capability. 

Our approach to the challenge guaranteed autonomy in the strictest sense: \textbf{for the generation of our solutions, there was absolutely no human intervention}. Humans experts inspected the final output of this pipeline for evaluation purposes only, without altering any content. We ran two different agents and designated one ``preferred solution'' per problem, whose ratings are displayed in Table \ref{table:intro-results}. This designation admittedly draws upon our own expertise. 

 \textbf{Correctness.} We interpreted ``Correct'' as meaning ``publishable after minor revisions, within the established range of the peer review process'', consistent with the standards\footnote{\url{https://icarm.zulipchat.com/\#narrow/channel/568090-first-proof/topic/Mathematical.20standard/near/573992500}} voiced by the FirstProof authors. In particular, we do not claim that our solutions are publication-ready as originally generated. Many fail to meet the stated requirement that ``Citations should include precise statement numbers and should either be to articles published in peer-reviewed journals or to arXiv preprints'', but do meet the citation standards prevailing in the literature.

We emphasize that this is only our own interpretation of the challenge. Other reasonable interpretations exist, and the authors of \fp{} make clear in \citep{abouzaid2026proof} that it is not intended as a formal benchmark.

\section{Methodology and results}
We prompted the agent \aletheia{} from \citep{feng2026autonomousmathematicsresearch} with the problem statements from the \fp{} \LaTeX{} file, copy-pasted without any modification. The outputs of \aletheia{} were filtered, without any intermediate alteration, through a pre-determined verification and extraction prompt (exposed in \S \ref{sec:extraction-prompt}) designed to the stated standards of the \fp{} authors to produce ``a proof that conforms to the levels of rigor and scholarship prevailing in the mathematics literature.'' Moreover, the verification and extraction prompt elicited \LaTeX{} code directly as output, ensuring that manual intervention would not be required even for reformatting the response in a \LaTeX{} document. 

We then tried to judge the outputs of this pipeline, in some cases asking for help among our colleagues. In this process, we did not interact with the model at all--not even by prompting for clarification or elaboration on points we did not understand.  Our overall pipeline is illustrated below.

\includegraphics[scale=.5]{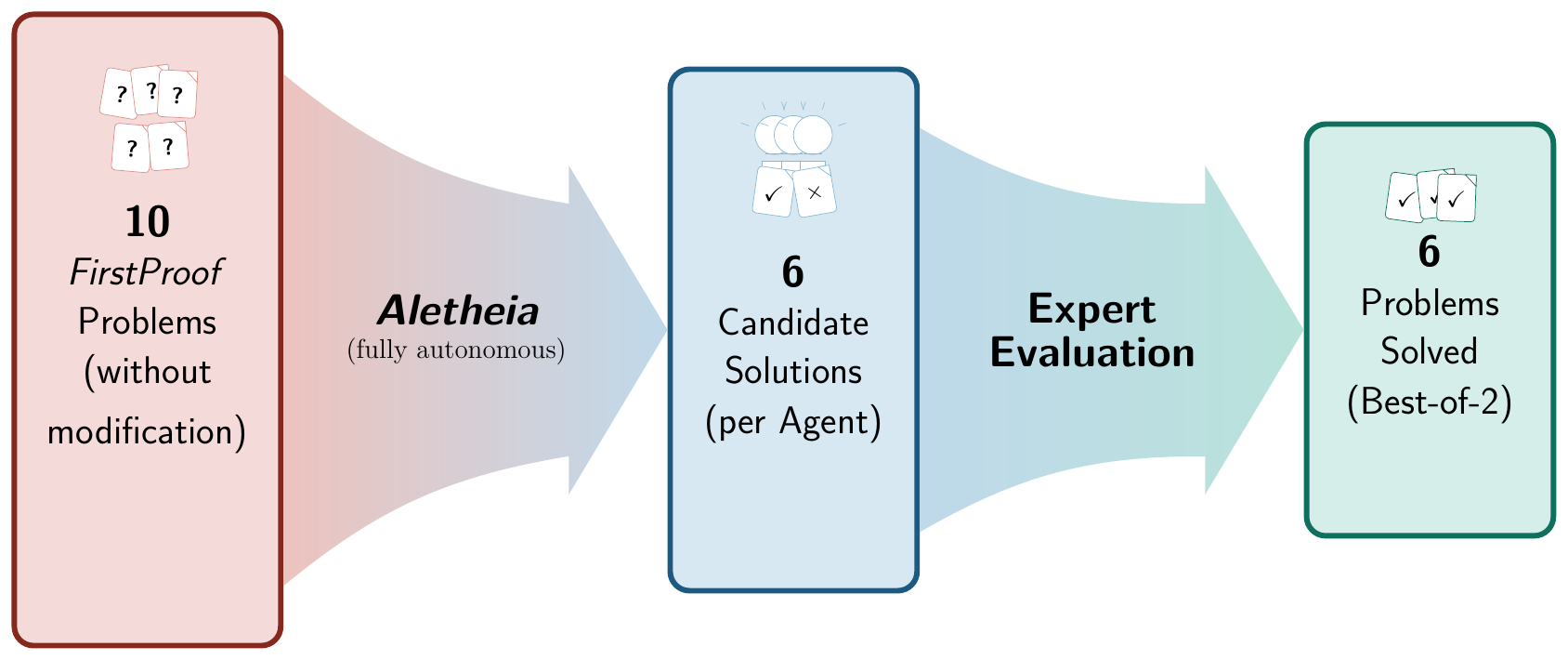}

For internal permissions reasons, we were not able to publicly release our results before official solutions were uploaded by the FirstProof authors on February 13, 11:59pm PST. In order to certify that our results were obtained without data contamination from these solutions, we e-mailed our solutions privately to the \fp{} authors at 11:07pm PST on February 13 (along with a preliminary version of this document including Table~\ref{table:results}, which represented our initial estimate of the correctness of the solutions at that point). We later shared\footnote{\url{https://icarm.zulipchat.com/\#narrow/channel/568090-first-proof/topic/Aletheia's.20solutions}} our solutions publicly on February 18, 9:27am PST and Mohammed Abouzaid (lead \fp{} author) confirmed, in the same thread, the existence of our solutions before the deadline.\footnote{Regrettably, there was one typo in our submission: the file labeled FP10\_A.pdf was instead for \agentb{} and should have been named FP10\_B.pdf; the submission for \agenta{} on FP\#10 was omitted, and is now included as FP10\_A.pdf.}

\subsection{\aletheia{} (Best of 2)}

We ran the agent \aletheia{} from \citep{feng2026autonomousmathematicsresearch} on two different base models. These will be designated as follows: 
\begin{enumerate}
    \item \agenta{}: with the same base model as Gemini 3 Deep Think \citep{gemini3deepthink2026} as of February 2026.
    \item \agentb{}: with the January 2026 base model of Gemini, referenced in \citep{feng2026autonomousmathematicsresearch}.
\end{enumerate}

\begin{table}[h]
\centering
\renewcommand{\arraystretch}{1.5} % Adjusts row height for readability
\begin{tabular}{|l|l|l|l|c|}
\hline
   & \textbf{\agenta{}} &  \textbf{\agentb{}} & \textbf{Zulip}\\ \hline
P1 & No output &  No output & \\ \hline
P2 & \color{green} Correct &  \color{green} Correct  &  \href{https://icarm.zulipchat.com/#narrow/channel/568090-first-proof/topic/Problem.202.20--.20Aletheia/with/574567015}{Link}\\ \hline
P3 & No output & No output &  \\ \hline
P4 & No output &  No output & \\ \hline
P5 & \color{green} Correct & {\color{red}Misinterpreted} & \href{https://icarm.zulipchat.com/#narrow/channel/568090-first-proof/topic/Problem.205.20--.20Aletheia/with/575042104}{Link} \\ \hline
P6 & No output &  No output & \\ \hline
P7 & \color{red}Critically Flawed &  \color{green} Correct & \href{https://icarm.zulipchat.com/#narrow/channel/568090-first-proof/topic/Problem.207.20--.20Aletheia/with/574990987}{Link} \\ \hline
P8 & \color{red} Inadequate   &  \color{green!70}Correct?  & \href{https://icarm.zulipchat.com/#narrow/channel/568090-first-proof/topic/Problem.208.20--.20Aletheia/with/574569368}{Link} \\ \hline
P9 & \color{green} Correct &  \color{green} Correct & \href{https://icarm.zulipchat.com/#narrow/channel/568090-first-proof/topic/Problem.209.20--.20Aletheia/with/574726804}{Link} \\ \hline
P10 & \color{green} Correct & \color{green} Correct & \href{https://icarm.zulipchat.com/#narrow/channel/568090-first-proof/topic/Problem.2010.20--.20Aletheia/with/574570445}{Link} \\ \hline
\end{tabular}
\caption{Our current ({\it post-deadline}) estimation of the results based on the consensus of expert assessments. On P8, the expert assessment was not unanimous. We include links with public comments (on Zulip) for individual problems. 
}
\label{table:resultsafter}
\end{table}

On the 10 \fp{} problems, our agents produced solution candidates to 6 problems (P2, P5, P7, P8, P9, P10). From a best-of-2 evaluation, the majority opinion of expert evaluations indicated that all 6 problems were solved correctly (under the interpretation of needing only minor revisions), although the assessments on P8 were not unanimous: only 5 out of 7 experts rated it Correct. The assessment of individual solutions is displayed in Table \ref{table:resultsafter}. Section \ref{ssec:evals} discusses the evaluations in more detail. 

For the other 4 problems (P1, P3, P4, P6) both of our agents returned no solution: either by explicitly outputting ``No solution found'', or by not returning any output within the time limit. This self-filtering feature was one of the key design principles of \aletheia{}; we view reliability as the primary bottleneck to scaling up AI assistance on research mathematics. We suspect that, given the limited bandwidth for human expert verification, many practicing researchers would prefer to trade raw problem-solving capability for increased accuracy.\footnote{This was our motivation for building \aletheia{}, and indeed the origin of the name.} 

\textbf{Inference cost.} The inference-time computation expended by \textit{Aletheia} on the \textit{FirstProof} problems can be interpreted as a rough proxy of the problem difficulty from the agent's perspective. In Figure \ref{fig:fp_compute}, we display the inference cost of each candidate solution as a multiple of the inference cost of the solution to Erd\H{o}s-1051 from (Feng et al., 2026a). Both base models used here are different from the one used in (Feng et al., 2026a), so the comparison is not on equal footing, but it gives some indication. For each problem, the inference cost exceeded that of Erd\H{o}s-1051.

For P7 in particular, the inference cost exceeded previously observed scales by an order of magnitude, both because the Generator subagent took much more computation to produce a candidate solution, and because more interactions were required to pass the Verifier subagent. We note that while most of the FirstProof problems were described as Lemmas arising in the recent research of the authors, P7 was advertised as an open problem in the book of Weinberger \citep{Wein23}, prior to its resolution by Cappell--Weinberger--Yan (not published until the \fp{} solutions). 

\begin{figure}[htbp]
    \centering
    \includegraphics[width=0.8\textwidth]{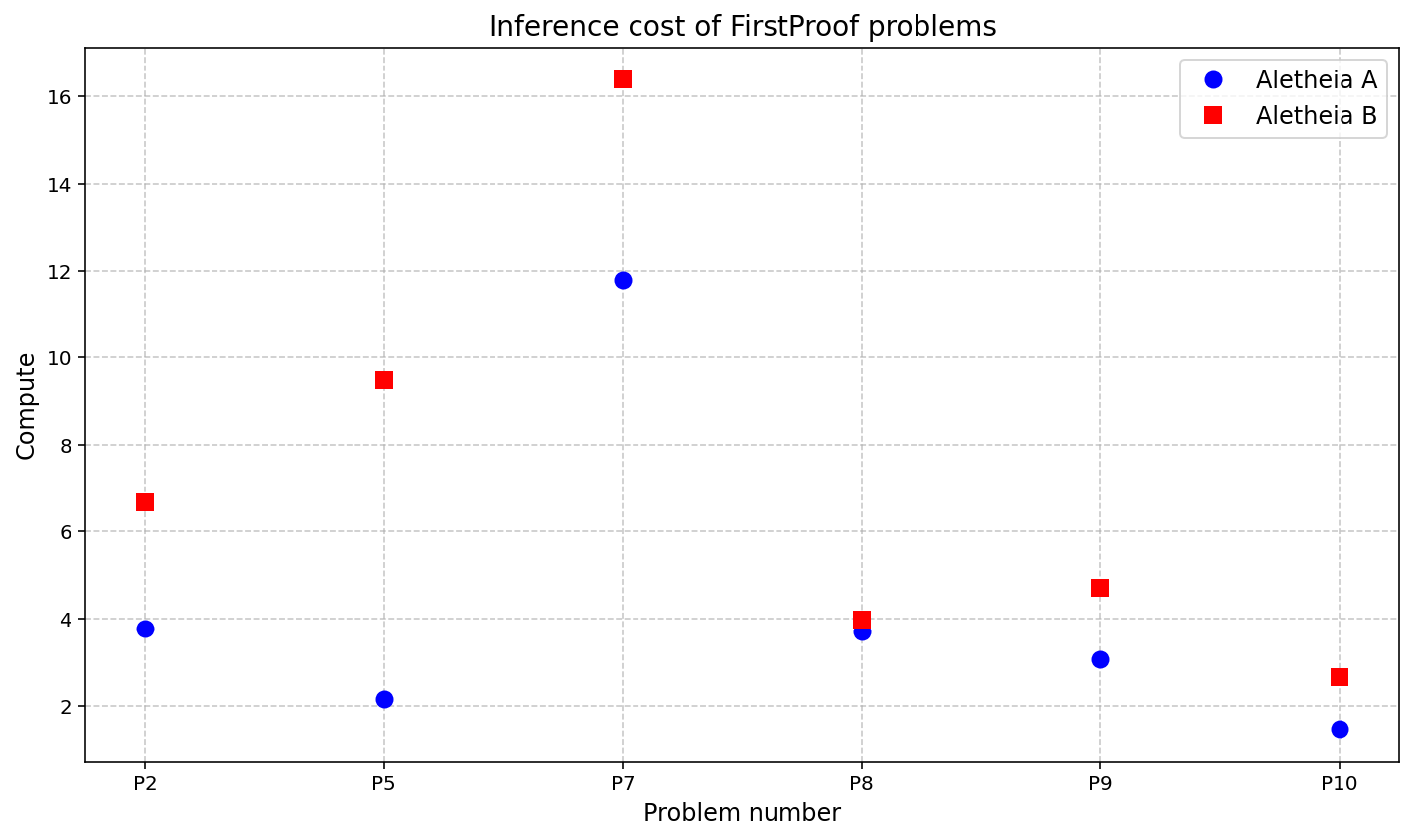}
    \caption{Plot of the inference cost per \fp{} problem, as a multiple of the inference cost of the solution to Erd\H{o}s-1051 from \citep{feng2026semiautonomousmathematicsdiscoverygemini}.}
    \label{fig:fp_compute}
\end{figure}

Not all of the problems required a large inference budget to solve. Aryan Mokhtari and David Woodruff succeeded in manually orchestrating the publicly available Gemini 3 Deep Think model to solve Problem 10, as described in Appendix~\ref{sec:public_dt_p10}. 

%---operates with less inference-time scaling per output. The larger inference cost multipliers seen for \textit{Aletheia} reflect its agentic orchestration: expending more compute on generation and verification to autonomously isolate and refine correct proofs. 

\subsection{Evaluations}\label{ssec:evals} 
To evaluate our outputs, we obtained independent feedback from at least two academic mathematicians (some of whom were partially affiliated with Google) for each problem. When the experts felt less confident, we solicited more opinions from academic mathematicians. Table \ref{table:resultsafter} summarizes our assessments. A problem-by-problem description is given below.

\textbf{P2.} Four out of four experts agreed that both solutions were Correct.  

\textbf{P5.} Experts pointed out that there was ambiguity in the formulation of the question. Four out of four experts agreed that \agenta{}'s solution is Correct. \agentb{} interpreted the ``slice filtration'' in an archaic way that differs from its modern usage. Because of this, reviewers classified \agentb{}'s solution as a Misinterpretation of the problem, and did not further vet it for mathematical correctness. 

\textbf{P7.} Three out of three experts agreed that \agentb{}'s solution is Correct. \agenta{}'s solution is Critically Flawed. It contains two arguments, both of which boil down to the claim that if $\sigma$ is an order 2 automorphism acting freely on a manifold $M$, then the (compactly supported) rational Euler characteristic of $M$ is divisible by $2$. The attempted justification invokes multiplicativity of (compactly supported) rational Euler characteristic, but this is not justified without appropriate finiteness assumptions on $M$; this fallacy is noted in the official problem comments.

\textbf{P8.} Experts deemed \agenta{}'s solution to P8 to be Inadequate. For \agentb{}'s solution to P8, three out of three external specialists in symplectic geometry judged it to be correct prior to the February 13 deadline. An internal mathematician expressed reservations, so we solicited more assessments, ultimately obtaining opinions from four specialists in symplectic geometry and three additional mathematicians with adjacent expertise. In total, three specialists and two adjacent mathematicians considered the solution to be Correct. A representative quote from this group was, ``Overall, while I wouldn't say this solution is perfect, I think it's reasonable to count it as a correct proof.'' The remaining specialist and adjacent mathematician considered the proof to be incomplete due to the level of detail. A representative quote from this group was, ``The shakiest part is indeed in the interpolation step when the local smoothings at the vertices of the polyhedral Lagrangian surface need to be extended to smoothings along the edges. I think it is fair to object that more detail is needed at this step, and this is true for the proof attempts provided by both agents.''

Upon examining the expert evaluations, we realized that all were essentially in agreement on the mathematical content, and the ambiguity came from subjective interpretation of whether the missing detail exceeded the threshold of ``minor revisions''. None of the experts expressed that there were errors in the argument, but most experts voiced that parts of Steps 3 and 4 were vague or sketchy (see \S \ref{ssec:evals}), and that the solution as a whole was not publishable without revisions. 

\textbf{P9.} Four out of four experts agreed that \agenta{}'s solution is correct. Two out of two experts agreed that \agentb{}'s solutions are correct.

\textbf{P10.} Two out of two experts agreed that both \agenta{}'s and \agentb{}'s solutions are correct.

% \subsection{Accuracy} 
\subsection{Further comparisons} 
\agenta{} and \agentb{} each produced candidate solutions for the same six problems. % out of ten. 
Each agent individually had at least one false positive, but their best-of-2 performance produced credible solutions to all six problems. This outcome shows a promising increase in accuracy over the December 2025 version of \aletheia{} used for the Erd\H{o}s problems in \citep{feng2026semiautonomousmathematicsdiscoverygemini}. Compared to that version, \agenta{} and \agentb{} featured improvements to both the agentic scaffolding and the base models.

% To provide additional context, we briefly comment on a related effort.

% Specifically, 
In addition to Aletheia, an independent evaluation of the publicly available Gemini 3 Deep Think model highlighted its strong capabilities. While not strictly autonomous---relying on two humans to sample and concatenate the best outputs ---this effort produced a solution to Problem 10 that {\it matches the optimal theoretical complexity bound also discovered autonomously by Aletheia A}, but with much less inference scaling; see Appendix~\ref{sec:public_dt_p10}.

% {\color{red}DW notes:
% \\\\
% (1) the early stop aletheia does correspond to the best DT solution up to some minor differences, and this does significantly beat human and OAI solutions since iteration count doesn't depend on q
% \\\\
% (2) the other DT solution is similar to OAI's solution, where the precondition is P itself to avoid expensive inversiona
% \\\\
% DW: will clean this up}

\section{Acknowledgments} We thank Daniel Alvarez-Gavela, Otis Chodosh, Vincent Cohen-Addad, Laurent Cote, Jim Davis, Alex Davies, Jim Fowler, Javier Gomez-Serrano, Bogdan Georgiev, Vineet Gupta, Euiwoong Lee, Gilad Lerman, Yaguang Li, Hanzhao (Maggie) Lin, Daniel Litt, Chi-Heng Lo, Aranyak Mehta, Mona Merling, Daniel Miao, Agustin Moreno, Danny Xiaolin Shi, George Tsoukalas, Allen Yuan, Yufei Zhao, Daniel Zheng, and Goran Zuzic for help. We are grateful to the Institute for Computer-Aided Reasoning in Mathematics under NSF grant DMS 2425401 for hosting online discussion of the FirstProof results. Thanks to Benoit Schillings, Koray Kavukcuoglu, Demis Hassabis, and Sergey Brin for support and for encouraging us to tackle harder problems.

\bibliography{references}

\appendix

\section{Verification and Extraction Prompt}\label{sec:extraction-prompt}
Below we display the verification and extraction prompt used on \aletheia{}'s outputs, which were run through Gemini 3 Deep Think. This prompt was designed elicit output of the quality and format requested by the FirstProof authors. In particular, it allowed us to obtain \LaTeX{} code directly as output, ensuring that manual intervention would not be required even for reformatting the response in a \LaTeX{} document. 

Running this prompt filtered out a P3 submission from \agentb{} with verdict [WRONG].  A [FIXABLE] verdict was returned on P5 and P7 from \agenta{}, and our logs display the (autonomously) revised outputs. Other outputs listed above were marked [CORRECT] by the extraction prompt, and no changes were made to produce the final output. 

\begin{lstlisting}
You are an expert peer reviewer for a top-tier academic journal. Your task is to rigorously evaluate a problem and its candidate solution. If you find a correct solution, output it as a latex document that conforms to the levels of rigor and scholarship prevailing in the mathematics literature.

Please approach the evaluation using the following structured process:

**1. Independent Verification**
Before evaluating the candidate, use your reasoning process to independently analyze the `<problem>` to determine the correct methodology and potential edge cases. Then, do a line-by-line verification of the `<candidate_solution>`. Actively search for logical fallacies, unstated assumptions, calculation errors, or lack of rigor.

After your reasoning process, format your final response exactly as follows:

### 1. Critique
Provide a concise summary of your analysis. Point out any specific flaws, leaps in logic, or informalities found in the candidate solution. The solution needs to conform to the levels of rigor and scholarship prevailing in the mathematics literature. If the solution cites the literature, carefully check that all citations include precise statement numbers and should either be to articles published in peer-reviewed journals or to arXiv preprints.

### 2. Verdict
Based on your critique, declare exactly ONE of the following verdicts in bold:
- **[CORRECT]**: The solution is flawless, completely rigorous, and requires no changes.
- **[WRONG]**: The solution is fundamentally flawed, relies on invalid logic, or cannot be salvaged without a complete rewrite of the core approach.
- **[FIXABLE]**: The core approach is sound, but it contains minor errors, skips necessary steps, or lacks formal academic rigor.

### 3. Resolution
Execute the corresponding action based on your verdict:
- If **[CORRECT]**: Briefly state why the solution meets publication standards.
- If **[WRONG]**: Explicitly detail the fatal flaw in the approach and mathematically/logically explain why it fails. (Do not write a new solution from scratch).
- If **[FIXABLE]**: Generate a **complete, corrected version** of the solution from start to finish. Do not merely list the fixes. The revision must be a cohesive, standalone proof/solution written at a level of completeness, clarity, and rigor suitable for peer-reviewed journal publication, conforming to the levels of rigor and scholarship prevailing in the mathematics literature. If the solution cites the literature, carefully check that all citations include precise statement numbers and should either be to articles published in peer-reviewed journals or to arXiv preprints.


<problem>
[INSERT PROBLEM HERE]
</problem>

<candidate_solution>
[INSERT CANDIDATE SOLUTION HERE]
</candidate_solution>
\end{lstlisting}

\section{Pre-deadline evaluations}\label{sec:results}

This Appendix records the ``best-of-2'' submission for each problem that was sent to the FirstProof authors on February 13, 2026. Originally, we ran another agent called \agenta{}f, a variant setting of \aletheia{} that overrode the natural orchestration to force greater inference expenditure during generation. \agenta{}f was aborted shortly after producing solutions to P2 and P9 to save inference cost. To preserve the best-of-2 nature of our attempt, we did not run P2, P9 (or P5) on \agentb{} until the evening of February 13, at which point we ran \agentb{} on those problems for future ablation purposes only (without reading the output).

Table~\ref{table:results} is reproduced exactly from the February 13 e-mail sent to the FirstProof authors. It displays the internal evaluations of the solutions as of February 13, 2026, and a ``first choice'' submission for each problem. \emph{The accuracy of these evaluations is superseded by Table~\ref{table:resultsafter}, which is based on more careful evaluations with a wider net of experts}; we include Table~\ref{table:results} only for historical reasons. Note, however, that the best-of-2 results from Table~\ref{table:results} and the best-of-2 results from Table~\ref{table:resultsafter} both align with Table~\ref{table:intro-results}. 

We discuss the differences between Table~\ref{table:results} and Table~\ref{table:resultsafter}.

\textbf{P2.} \agentb{}'s submission for P2, left unread before February 14, was later found to be \textbf{Correct}. 

\textbf{P5.} \agentb{}'s submission for P5, left unread before February 14, was later found to be based on a \textbf{Misinterpretation} of the question. 

\textbf{P7.} We were initially unable to validate \agentb{}'s solution to P7 due to its late appearance, as well as limited internal expertise. Based on prior experience with solutions of such complexity, we conservatively guessed that it was Incorrect. Later, upon closer examination, we realized that \agentb{}'s solution to P7 was in fact \textbf{Correct}.

\textbf{P8.} Prior to February 13, three out of three experts deemed \agenta's submission for P8 to be Correct (based on quick appraisals). Upon closer examination, at least one expert commented that ``it is sketchy in some important places (I think it’s fair to say there are some gaps, arguably errors), so does not reach the typical standard of a correct proof ''. Another expert said ``actually, maybe [\agenta{}'s solution] is a bit short of correct, but I think it is not far off''. We therefore changed the verdict to \textbf{Inadequate} in Table~\ref{table:resultsafter}. 

\textbf{P9.} \agentb{}'s solution to P9, left unread before February 14, was later found to be \textbf{Correct}. 

\textbf{P10.} \agenta{}'s solution to P10 was labeled Incorrect in Table \ref{table:results}. This turned out to be a miscommunication from one expert to the lead author, and after reassessment experts were unanimous that the original solution was correct. In fact, they considered it to be more optimal than \agentb{}'s solution, as Aletheia A autonomously derived a precomputation step that removes the $O(q)$ dependency from the iterative PCG loop. While we present Aletheia B's solution in Appendix~\ref{subsect:p10} as it was our initial ``first choice'' submission, the full Aletheia A solution is available online \href{https://icarm.zulipchat.com/#narrow/channel/568090-first-proof/topic/Problem.2010.20--.20Aletheia/with/574570445}{here}.

\begin{table}[h]
\centering
\renewcommand{\arraystretch}{1.5} % Adjusts row height for readability
\begin{tabular}{|l|l|l|l|}
\hline
\multicolumn{4}{|c|}{\textbf{Deprecated (pre-deadline) evaluations}} \\ \hline
   & \textbf{\agenta{}} & \textbf{\agenta{}f} & \textbf{\agentb{}} \\ \hline
P1 & N/A & N/A & N/A \\ \hline
P2 & \color{green}Correct & Correct & ? \\ \hline
P3 & N/A & N/A & \color{orange}Filtered out \\ \hline
P4 & N/A & N/A & N/A \\ \hline
P5 & \color{green}Correct & N/A &  ? \\ \hline
P6 & N/A & N/A & N/A \\ \hline
P7 & \color{red}Incorrect & N/A & Incorrect? \\ \hline
P8 & Correct & N/A & \color{green}Correct \\ \hline
P9 & Correct & \color{green}Correct &  ? \\ \hline
P10 & \color{red}Incorrect & N/A & \color{green}Correct \\ \hline
\end{tabular}
\caption{Our pre-deadline estimation of the results, as sent to the \fp{} authors at 11:07pm PST on February 13. {\color{green}Correct} = our first choice solution. {\color{red}Incorrect} = we are confident it is incorrect. {\color{orange}Filtered out} = filtered by the additional verification and extraction prompt. ? = ran for ablation purposes on Friday evening; not read before the deadline. \agenta{} is \aletheia{} with Gemini 3 Deep Think as base model. \agenta{}f is an extended generation variant of \agenta{}. \agentb{} is \aletheia{} with base model B, from \S 2 of \citep{feng2026autonomousmathematicsresearch}. `N/A' indicates either \aletheia{} returned ``No solution found'' or failed to return a solution within the time limit.}
\label{table:results}
\end{table}

\section{Raw prompts and outputs}
For each FirstProof problem answered by at least one of our two agents, we display the raw prompt and ``first choice'' output, according to our original evaluation in Table \ref{table:results}, which was as sent to the FirstProof authors prior to the deadline. We also include remarks with selected expert commentary. 

\subsection{Problem 2}

\begin{problem}{FirstProof \# 2}
    Let \(F\) be a non-archimedean local field with ring of integers \(\mathfrak o\).  Let $N_r$ denote the subgroup of $\mathrm{GL}_{r}(F)$ consisting of upper-triangular unipotent elements.  Let \(\psi:F\to \mathbb C^\times\) be a nontrivial additive character of conductor \(\mathfrak o\), identified in the standard way with a generic character of $N_r$.
Let \(\Pi\) be a generic irreducible admissible representation of \(\mathrm{GL}_{n + 1}(F)\), realized in its \(\psi^{-1}\)-Whittaker model \(\mathcal W(\Pi,\psi^{-1})\).  Must there exist \(W\in \mathcal W(\Pi,\psi^{-1})\) with the following property?

Let $\pi$ be a generic irreducible admissible representation of \(\mathrm{GL}_{n}(F)\), realized in its $\psi$-Whittaker model \(\mathcal W(\pi,\psi)\).  Let $\mathfrak{q}$ denote the conductor ideal of $\pi$, let \(Q\in F^\times\) be a generator of \(\mathfrak q^{-1}\), and set
\[
  u_Q := I_{n+1} + Q\,E_{n,n+1} \in \mathrm{GL}_{n + 1}(F),
\]
where \(E_{i, j}\) is the matrix with a \(1\) in the \((i, j)\)-entry and \(0\) elsewhere.  For some \(V\in \mathcal W(\pi,\psi)\), the local Rankin--Selberg integral
\[
  \int_{N_n\backslash \mathrm{GL}_{n}(F)} W(\operatorname{diag}(g,1) u_Q)\,V(g)\,|\det g|^{s-\frac12}\,dg
\]
is finite and nonzero for all \(s\in\mathbb C\).
\end{problem}

\begin{solution}{Model Response (Agent A)}
    
\section*{Introduction}
Let $F$ be a non-archimedean local field with ring of integers $\mathfrak{o}$ and maximal ideal $\mathfrak{p}$. We fix a uniformizer $\varpi \in \mathfrak{p}$ and let $q = |\mathfrak{o}/\mathfrak{p}|$. Let $N_r$ denote the subgroup of $\mathrm{GL}_r(F)$ consisting of upper-triangular unipotent elements, and let $K_r = \mathrm{GL}_r(\mathfrak{o})$ denote the standard maximal compact subgroup. Let $\psi: F \to \mathbb{C}^\times$ be a fixed nontrivial additive character of conductor $\mathfrak{o}$. 

\begin{theorem}
Let $\Pi$ be a generic irreducible admissible representation of $\mathrm{GL}_{n+1}(F)$. There exists $W \in \mathcal{W}(\Pi, \psi^{-1})$ such that for any generic irreducible admissible representation $\pi$ of $\mathrm{GL}_n(F)$ with conductor ideal $\mathfrak{q}$, and any generator $Q \in F^\times$ of $\mathfrak{q}^{-1}$, there exists $V \in \mathcal{W}(\pi, \psi)$ such that the local Rankin--Selberg integral
\[
  Z(s, W, V) = \int_{N_n\backslash \mathrm{GL}_{n}(F)} W(\operatorname{diag}(g,1) u_Q)\,V(g)\,|\det g|^{s-\frac12}\,dg
\]
is finite and nonzero for all $s\in\mathbb C$, where $u_Q = I_{n+1} + Q\,E_{n,n+1}$.
\end{theorem}

\begin{proof}
\textbf{Step 1: Gelfand--Kazhdan restriction and $s$-independence.}

We evaluate the right translation of $W$ by $u_Q$. Block matrix multiplication yields
\[
  \operatorname{diag}(g, 1) u_Q = U_Q \operatorname{diag}(g, 1), \quad \text{where} \quad U_Q = I_{n+1} + Q \sum_{i=1}^n g_{i,n} E_{i, n+1}.
\]
Because $U_Q \in N_{n+1}$ and its only superdiagonal entry (i.e., immediately above the main diagonal) is located at $(n, n+1)$ with value $Q g_{n,n}$, the left $N_{n+1}$-equivariance of the Whittaker model $\mathcal{W}(\Pi, \psi^{-1})$ entails
\[
  W(\operatorname{diag}(g, 1) u_Q) = \psi^{-1}(Q g_{n,n}) W(\operatorname{diag}(g, 1)) = \psi(-Q g_{n,n}) W(\operatorname{diag}(g, 1)).
\]
By the Gelfand--Kazhdan restriction theory for the Kirillov model (H.\ Jacquet, I.\ I.\ Piatetski-Shapiro, and J.\ A.\ Shalika, \textit{Rankin--Selberg Convolutions}, 1983), the restriction map $\left.W \mapsto W\right|_{\mathrm{GL}_n(F)}$ surjects onto a space of functions containing $\mathcal{C}_c^\infty(N_n \backslash \mathrm{GL}_n(F), \psi^{-1})$. 

We define a smooth cut-off function $\Phi \in \mathcal{C}_c^\infty(N_n \backslash \mathrm{GL}_n(F), \psi^{-1})$ supported precisely on the open and closed double coset $N_n K_n$ by setting $\Phi(n k) = \psi^{-1}(n)$ for $n \in N_n$ and $k \in K_n$, and extending it by zero elsewhere. This is well-defined because $\psi$ is trivial on the intersection $N_n \cap K_n = N_n \cap \mathrm{GL}_n(\mathfrak{o})$. We fix a choice of $W \in \mathcal{W}(\Pi, \psi^{-1})$ satisfying $W(\operatorname{diag}(g, 1)) = \Phi(g)$.

Substituting this test vector $W$ restricts the domain of integration strictly to the compact quotient $N_n \backslash N_n K_n \simeq (N_n \cap K_n) \backslash K_n$. For $k \in K_n$, we have $|\det k| = 1$, which completely eliminates the complex parameter $s$. Normalizing the quotient measure appropriately, the integral converges absolutely to a finite, $s$-independent functional:
\[
  L_Q(V) = \int_{K_n} \psi(-Q k_{n,n}) V(k) \, dk.
\]
We are reduced to showing that there exists $V \in \mathcal{W}(\pi, \psi)$ such that $L_Q(V) \neq 0$.

\textbf{Step 2: The unramified case ($c=0$).}

If $\pi$ is unramified, its conductor ideal is $\mathfrak{q} = \mathfrak{o}$, meaning $Q \in \mathfrak{o}^\times$. We evaluate the functional on the normalized spherical vector $V = V_0$, which satisfies $V_0(k) = 1$ for all $k \in K_n$. Since $k_{n,n} \in \mathfrak{o}$ and $Q \in \mathfrak{o}^\times$, we have $-Q k_{n,n} \in \mathfrak{o}$. Because the additive character $\psi$ has conductor $\mathfrak{o}$, it follows that $\psi(-Q k_{n,n}) = 1$. The functional thus yields $L_Q(V_0) = \operatorname{vol}(K_n) > 0$.

\textbf{Step 3: Finite Fourier analysis setup ($c \ge 1$).}

Assume $\pi$ has conductor $\mathfrak{q} = \mathfrak{p}^c$ with $c \ge 1$. Here, $Q = \alpha \varpi^{-c}$ for some unit $\alpha \in \mathfrak{o}^\times$. Let $V_0 \in \mathcal{W}(\pi, \psi)$ be the essential newform, properly normalized so that $V_0(I_n) = 1$. 

Suppose, for the sake of contradiction, that $L_Q(\pi(h^{-1}) V_0) = 0$ for all $h \in K_n$. Evaluating the functional and making the change of variables $k \mapsto k h$ yields:
\[
  \int_{K_n} \psi(-Q (k h)_{n,n}) V_0(k) \, dk = 0 \quad \text{for all } h \in K_n.
\]
Let $e_n = (0, \dots, 0, 1)$ be viewed as a row vector in $\mathfrak{o}^n$. Then $(k h)_{n,n} = e_n k h e_n^T = (e_n k) y$, where $y = h e_n^T$ is a column vector. As $h$ traverses $K_n$, the vector $y$ traverses all unimodular column vectors in $\mathfrak{o}^n$. 

Because $Q = \alpha \varpi^{-c}$, the value $\psi(-Q(e_n k) y)$ depends on the row vector $e_n k$ exclusively modulo $\mathfrak{p}^c$. We descend to the finite quotient module $G = (\mathfrak{o}/\mathfrak{p}^c)^n$ by defining a function $H : G \to \mathbb{C}$ as follows:
\[
  H(\eta) = \int_{\{k \in K_n : e_n k \equiv \eta \pmod{\mathfrak{p}^c}\}} V_0(k) \, dk.
\]
If $\eta$ does not lift to a unimodular vector in $\mathfrak{o}^n$, the domain of integration is empty, forcing $H(\eta) = 0$. The vanishing assumption dictates that the finite Fourier transform of $H$ is identically zero on all unimodular vectors $y \in G$:
\[
  \widehat{H}(y) = \sum_{\eta \in G} H(\eta) \psi(-Q \eta y) = 0.
\]

\textbf{Step 4: Fourier inversion and translation invariance.}

Since $\widehat{H}(y) = 0$ for all unimodular $y$, the support of $\widehat{H}$ is restricted to non-unimodular vectors. Over the finite module $G$, a vector is non-unimodular if and only if all its entries belong to $\mathfrak{p}/\mathfrak{p}^c$, meaning $\widehat{H}$ is supported entirely on $\mathfrak{p}G$.

Applying the Fourier inversion formula over $G$, we obtain:
\[
  H(\eta) = \frac{1}{|G|} \sum_{y \in \mathfrak{p}G} \widehat{H}(y) \psi(Q \eta y).
\]
Let $y \in \mathfrak{p}G$, guaranteeing $y = \varpi z$ for some column vector $z \in (\mathfrak{o}/\mathfrak{p}^{c-1})^n$. For an arbitrary shift $\delta \in \mathfrak{p}^{c-1}G$, we may write $\delta = \varpi^{c-1} x$ with a row vector $x \in G$. The inner product computes as:
\[
  Q \delta y = (\alpha \varpi^{-c}) (\varpi^{c-1} x) (\varpi z) = \alpha x z \in \mathfrak{o}.
\]
Because the additive character $\psi$ has conductor $\mathfrak{o}$, we deduce $\psi(Q \delta y) = \psi(\alpha x z) = 1$. The additive character thus neutralizes the shift, giving $\psi(Q(\eta + \delta) y) = \psi(Q \eta y)$. Consequently, $H(\eta)$ exhibits translation invariance:
\[
  H(\eta + \delta) = H(\eta) \quad \text{for all } \delta \in \mathfrak{p}^{c-1}G.
\]

\textbf{Step 5: Level lowering and contradiction.}

We evaluate $H(e_n)$. The domain of integration enforces the condition $e_n k \equiv e_n \pmod{\mathfrak{p}^c}$, which strictly characterizes the mirabolic congruence subgroup $K_1(\mathfrak{p}^c)$ consisting of matrices in $K_n$ whose last row is congruent to $e_n$ modulo $\mathfrak{p}^c$. By foundational theory (H.\ Jacquet, I.\ I.\ Piatetski-Shapiro, and J.\ A.\ Shalika, \textit{Conducteur des repr\'esentations du groupe lin\'eaire}, 1981), the essential newform $V_0$ is invariant under right translation by $K_1(\mathfrak{p}^c)$. Since $V_0(I_n) = 1$, we obtain:
\[
  H(e_n) = \int_{K_1(\mathfrak{p}^c)} V_0(k) \, dk = \operatorname{vol}(K_1(\mathfrak{p}^c)) > 0.
\]
By the translation invariance established above, summing $H$ over the $q^n$ mutually distinct shifts $\delta \in \mathfrak{p}^{c-1}G$ yields:
\[
  \sum_{\delta \in \mathfrak{p}^{c-1}G} H(e_n + \delta) = q^n H(e_n) > 0.
\]
Conversely, returning to the definition of $H$, the disjoint union of the integration domains $\{k \in K_n : e_n k \equiv e_n + \delta \pmod{\mathfrak{p}^c}\}$ systematically aggregates to the coarser condition $e_n k \equiv e_n \pmod{\mathfrak{p}^{c-1}}$. This aggregated domain precisely identifies the broader mirabolic congruence subgroup $K_1(\mathfrak{p}^{c-1})$. Therefore:
\[
  \sum_{\delta \in \mathfrak{p}^{c-1}G} H(e_n + \delta) = \int_{K_1(\mathfrak{p}^{c-1})} V_0(k) \, dk.
\]
This integral calculates the evaluation at $I_n$ of the locally averaged vector 
\[
  v = \int_{K_1(\mathfrak{p}^{c-1})} \pi(k) V_0 \, dk \in \mathcal{W}(\pi, \psi).
\]
Because the conductor of $\pi$ is exactly $\mathfrak{p}^c$, the representation $\pi$ contains no nonzero vectors invariant under the larger mirabolic congruence subgroup $K_1(\mathfrak{p}^{c-1})$. Thus, the vector $v$ identically vanishes, mandating $v(I_n) = 0$.

We are thereby led to the contradiction $0 = q^n \operatorname{vol}(K_1(\mathfrak{p}^c)) > 0$. This contradiction demonstrates that the uniform vanishing hypothesis must be false. We conclude that there exists some translate $V = \pi(h^{-1}) V_0$ for which the local Rankin--Selberg integral is finite and nonzero, completing the proof.
\end{proof}
\end{solution}

\begin{remark} The response's citation to (H.\ Jacquet, I.\ I.\ Piatetski-Shapiro, and J.\ A.\ Shalika, \textit{Rankin--Selberg Convolutions}, 1983) should be refined to specifically (2.2) of that paper. The response's citation to (H.\ Jacquet, I.\ I.\ Piatetski-Shapiro, and J.\ A.\ Shalika, \textit{Conducteur des repr\'esentations du groupe lin\'eaire}, 1981) should be refined to specifically Théorème (5.1) of that paper. Other than these minor points, the solution is flawless. 
\end{remark}

\subsection{Problem 5}

\begin{problem}{FirstProof \# 5} Fix a finite group $G$.  Let $\aO$ denote an incomplete transfer
system associated to an $N_\infty$ operad.  Define the slice
filtration on the $G$-equivariant stable category adapted to $\aO$ and
state and prove a characterization of the $\aO$-slice connectivity of
a connective $G$-spectrum in terms of the geometric fixed points.
\end{problem}

\begin{solution}{Model Response (Agent A)}
\setcitestyle{numbers,round}
\section*{Setup and Structural Invariants}

In equivariant stable homotopy theory, commutative ring spectra can be parameterized by $N_\infty$ operads, which specify the sets of admissible norm maps. By the foundational work of Blumberg and Hill \cite[Section 4]{BlumbergHill2015}, and its subsequent combinatorial classification by Rubin \cite[Theorem 3.7]{Rubin2020} via indexing systems, the data of an $N_\infty$ operad on a finite group $G$ can be equivalently captured by a partial order formalism known as an \emph{incomplete transfer system}. This formulation was explicitly introduced and formalized by Balchin, Barnes, and Roitzheim \cite[Definition 2.4]{BBR2021}.

\begin{definition}[Balchin, Barnes, and Roitzheim {\cite[Definition 2.4]{BBR2021}}]
A \emph{transfer system} $\mathcal{O}$ on a finite group $G$ is a partial order $\subseteq_{\mathcal{O}}$ on the set of subgroups of $G$ that refines inclusion and satisfies two axioms:
\begin{enumerate}
\item \textbf{Conjugation:} $K \subseteq_{\mathcal{O}} H \implies gKg^{-1} \subseteq_{\mathcal{O}} gHg^{-1}$ for all $g \in G$.
\item \textbf{Restriction:} If $K \subseteq_{\mathcal{O}} H$ and $J \le H$, then $K \cap J \subseteq_{\mathcal{O}} J$.
\end{enumerate}
\end{definition}

Because the set of subgroups $\{ K \le H \mid K \subseteq_{\mathcal{O}} H \}$ is finite and closed under intersection (via the restriction axiom and the transitivity of the partial order), it possesses a unique minimal element, which we denote by $H_{\mathcal{O}}$. This allows us to define a natural index for subgroups adapted to the operad.

\begin{definition}
The \emph{$\mathcal{O}$-index} of a subgroup $J \le G$ is defined as the maximal index of an $\mathcal{O}$-admissible subgroup of $J$, which evaluates to:
\[ \| J \|_{\mathcal{O}} := [J : J_{\mathcal{O}}]. \]
\end{definition}

\begin{definition}
A finite $H$-set $T$ is \emph{$\mathcal{O}$-admissible} if the stabilizer of every point $t \in T$ is an $\mathcal{O}$-admissible subgroup of $H$ (i.e., $\operatorname{Stab}_H(t) \subseteq_{\mathcal{O}} H$). A permutation representation is $\mathcal{O}$-admissible if it is isomorphic to $\mathbb{R}[T]$ for an $\mathcal{O}$-admissible $H$-set $T$. Let $RO_{\mathcal{O}}^+(H)$ denote the monoid of isomorphism classes of $\mathcal{O}$-admissible permutation representations of $H$.
\end{definition}

We adapt the regular slice filtration (cf.\ Hill, Hopkins, and Ravenel \cite[Section 4.1]{HHR2016}; Hill and Yarnall \cite[Section 2]{HillYarnall2018}) to the incomplete transfer system $\mathcal{O}$ as follows:

\begin{definition}
For an integer $n \ge 0$, the \emph{$\mathcal{O}$-slice category} $\Sigma_{\ge n}^{\mathcal{O}}$ is the full subcategory of connective genuine $G$-spectra generated (under arbitrary wedges, extensions, and homotopy colimits) by the $\mathcal{O}$-slice cells:
\[ \left\{ G_+ \wedge_H S^V \;\middle|\; H \le G, \, V \in RO_{\mathcal{O}}^+(H), \text{ and } \dim_{\mathbb{R}}(V) \ge n \right\}. \]
\end{definition}

\section*{The Main Theorem}

We generalize the characterization of slice connectivity from Hill and Yarnall \cite[Theorem 2.5]{HillYarnall2018}, providing a sharp equivalence between the $\mathcal{O}$-slice filtration and the connectivity of geometric fixed points.

\begin{theorem}
\label{thm:main}
Let $\mathcal{O}$ be an incomplete transfer system on a finite group $G$, and let $n \ge 0$. A connective $G$-spectrum $X$ belongs to the $\mathcal{O}$-slice category $\Sigma_{\ge n}^{\mathcal{O}}$ if and only if for every subgroup $J \le G$, the geometric fixed point spectrum $\Phi^J(X)$ is $\lceil n / \| J \|_{\mathcal{O}} \rceil$-connective (i.e., it belongs to the non-equivariant category $\mathrm{Sp}_{\ge \lceil n / \| J \|_{\mathcal{O}} \rceil}$).
\end{theorem}

\section*{A Combinatorial Lemma}

We first establish a strict lower bound on the fixed-point dimensions of $\mathcal{O}$-admissible representations.

\begin{lemma}
\label{lem:dim_bound}
For any $H \le G$, $V \in RO_{\mathcal{O}}^+(H)$, and $J \le H$, the dimension of the $J$-fixed points satisfies
\[ \dim(V^J) \ge \frac{\dim(V)}{\| J \|_{\mathcal{O}}}. \]
\end{lemma}

\begin{proof}
By additivity over disjoint unions of $H$-orbits, it suffices to prove this for transitive admissible representations $V = \mathbb{R}[H/K]$ where $K \subseteq_{\mathcal{O}} H$. The dimension $\dim(V^J)$ is precisely the number of $J$-orbits on the coset space $H/K$, which corresponds to the number of double cosets $|J \backslash H / K|$.

Consider a generic orbit corresponding to the double coset $JxK$. The stabilizer in $J$ of the coset $xK \in H/K$ is $L = J \cap xKx^{-1}$.
By the conjugation axiom, $xKx^{-1} \subseteq_{\mathcal{O}} xHx^{-1} = H$. By the restriction axiom applied to $J \le H$, we obtain $L \subseteq_{\mathcal{O}} J$. Because $L$ is $\mathcal{O}$-admissible in $J$, its index satisfies $[J : L] \le [J : J_{\mathcal{O}}] = \| J \|_{\mathcal{O}}$.

The size of this generic $J$-orbit on $H/K$ is $[J:L]$, which is bounded above by $\| J \|_{\mathcal{O}}$. Partitioning the elements of $H/K$ into these orbits yields:
\[ \dim(V) = [H : K] = \sum_{\text{orbits}} [J : L] \le \sum_{\text{orbits}} \| J \|_{\mathcal{O}} = \dim(V^J) \cdot \| J \|_{\mathcal{O}}. \]
Dividing by $\| J \|_{\mathcal{O}}$ yields the desired inequality.
\end{proof}

\section*{Proof of Necessity}

Assume $X \in \Sigma_{\ge n}^{\mathcal{O}}$. Since the geometric fixed point functor $\Phi^J$ is exact and preserves arbitrary wedges and homotopy colimits, it suffices to verify the connectivity condition on the generators $Y = G_+ \wedge_H S^V$ with $\dim_{\mathbb{R}}(V) \ge n$.

By the geometric double coset Mackey formula:
\[ \Phi^J(Y) \simeq \Phi^J(G_+ \wedge_H S^V) \simeq \bigvee_{x \in J \backslash G / H, \; J \le x H x^{-1}} S^{(xV)^J}. \]
For each double coset $x$ contributing to the wedge, define $J' = x^{-1} J x \le H$. The space $(xV)^J$ is naturally isomorphic to $V^{J'}$. By Lemma \ref{lem:dim_bound}, we have $\dim(V^{J'}) \ge \dim(V) / \| J' \|_{\mathcal{O}}$.

Because the transfer system $\mathcal{O}$ is conjugation-invariant, conjugation by $x$ induces a poset isomorphism between the $\mathcal{O}$-admissible subgroups of $J'$ and those of $J$, giving $\| J' \|_{\mathcal{O}} = \| J \|_{\mathcal{O}}$. Thus, $\dim(V^{J'}) \ge n / \| J \|_{\mathcal{O}}$. Because dimensions of real representations are integers, we obtain the bound $\dim(V^{J'}) \ge \lceil n / \| J \|_{\mathcal{O}} \rceil$.

Therefore, $\Phi^J(Y)$ is a wedge of non-equivariant spheres of dimension at least $\lceil n / \| J \|_{\mathcal{O}} \rceil$, implying $\Phi^J(Y) \in \mathrm{Sp}_{\ge \lceil n / \| J \|_{\mathcal{O}} \rceil}$. Because the subcategory $\mathrm{Sp}_{\ge \lceil n / \| J \|_{\mathcal{O}} \rceil}$ is closed under wedges, extensions, and homotopy colimits, and $\Phi^J$ preserves these operations, this connectivity bound necessarily extends from the generators $Y$ to the generic spectrum $X$. This completes the proof of necessity.

\section*{Proof of Sufficiency}

We proceed by strong induction on the order of $G$. The base case $G=\{e\}$ is trivial. Assume the theorem holds for all proper subgroups of $G$. Let $X$ be a $G$-spectrum satisfying the fixed-point bounds for all $J \le G$.

Let $\mathcal{P}$ denote the family of proper subgroups of $G$. Consider the $G$-isotropy separation cofiber sequence:
\[ E\mathcal{P}_+ \wedge X \to X \to \tilde{E}\mathcal{P} \wedge X. \]
Since the subcategory $\Sigma_{\ge n}^{\mathcal{O}}$ is closed under extensions, it suffices to show that both outer terms belong to $\Sigma_{\ge n}^{\mathcal{O}}$.

\subsection*{Step 1: The term $E\mathcal{P}_+ \wedge X$}
The classifying space $E\mathcal{P}$ is a $G$-CW complex, and thus $E\mathcal{P}_+$ is equipped with a skeletal filtration where the cofibers of the successive skeletal inclusions $E\mathcal{P}_+^{(k-1)} \to E\mathcal{P}_+^{(k)}$ are finite wedges of cells of the form $(G/H)_+ \wedge S^k$ for proper subgroups $H \in \mathcal{P}$ and $k \ge 0$. Smashing this filtration with $X$, we deduce that $E\mathcal{P}_+ \wedge X$ is built via wedges, extensions, and homotopy colimits from spectra of the form $(G/H)_+ \wedge S^k \wedge X \simeq \Sigma^k (G_+ \wedge_H i_H^* X)$. Because $\Sigma_{\ge n}^{\mathcal{O}}$ is closed under homotopy colimits, it is closed under suspensions (as $\Sigma Z$ is the homotopy colimit of $\ast \leftarrow Z \rightarrow \ast$). Establishing that the unsuspended spectrum $G_+ \wedge_H i_H^* X \in \Sigma_{\ge n}^{\mathcal{O}}$ is strictly sufficient to guarantee that $E\mathcal{P}_+ \wedge X \in \Sigma_{\ge n}^{\mathcal{O}}$.

For any proper subgroup $H < G$, let $\mathcal{O}_H$ denote the transfer system restricted to subgroups of $H$. For any $J \le H$, we note $\Phi^J(i_H^* X) \simeq \Phi^J(X)$. By hypothesis, this spectrum is $\lceil n / \| J \|_{\mathcal{O}} \rceil$-connective. Because the minimal admissible subgroups in $\mathcal{O}_H$ structurally coincide with those in $\mathcal{O}$, we have $\| J \|_{\mathcal{O}_H} = \| J \|_{\mathcal{O}}$. Thus, $i_H^* X$ perfectly satisfies the connectivity hypothesis for $\mathcal{O}_H$. By the inductive hypothesis, $i_H^* X \in \Sigma_{\ge n}^{\mathcal{O}_H}$. 

The exact induction functor $G_+ \wedge_H (-)$ sends $\mathcal{O}_H$-admissible cells identically to $\mathcal{O}$-admissible cells. Specifically, for any $K \le H$ and any $\mathcal{O}_H$-admissible representation $W$ of $K$, we have $G_+ \wedge_H (H_+ \wedge_K S^W) \simeq G_+ \wedge_K S^W$. Because $W$ is $\mathcal{O}_H$-admissible, its stabilizers are $\mathcal{O}_H$-admissible in $K$, which natively implies they are $\mathcal{O}$-admissible. Therefore, $G_+ \wedge_H i_H^* X \in \Sigma_{\ge n}^{\mathcal{O}}$, which guarantees $E\mathcal{P}_+ \wedge X \in \Sigma_{\ge n}^{\mathcal{O}}$.

\subsection*{Step 2: The term $\tilde{E}\mathcal{P} \wedge X$}
By standard geometric localization (see, e.g., \cite[Lemma 2.2]{HillYarnall2018} or \cite[Section 2.5]{HHR2016}), $\tilde{E}\mathcal{P} \wedge X$ is a geometric $G$-spectrum naturally determined by its geometric $G$-fixed points, yielding the equivalence $\tilde{E}\mathcal{P} \wedge X \simeq \tilde{E}\mathcal{P} \wedge \operatorname{infl}(\Phi^G X)$, where $\operatorname{infl}$ denotes inflation from non-equivariant spectra.

By hypothesis, $\Phi^G X \in \mathrm{Sp}_{\ge m}$, where $m = \lceil n / \| G \|_{\mathcal{O}} \rceil$. Consequently, $\Phi^G X$ is generated (under wedges, extensions, and homotopy colimits) by non-equivariant spheres $S^k$ for $k \ge m$. Because these operations commute with inflation and smashing with $\tilde{E}\mathcal{P}$, it suffices to prove that $\tilde{E}\mathcal{P} \wedge \operatorname{infl}(S^k) \simeq \tilde{E}\mathcal{P} \wedge S^k \in \Sigma_{\ge n}^{\mathcal{O}}$ for all integers $k \ge m$.

Consider the regular admissible permutation representation $V = k \mathbb{R}[G/G_{\mathcal{O}}]$. Its dimension is $k [G : G_{\mathcal{O}}] = k \| G \|_{\mathcal{O}} \ge m \| G \|_{\mathcal{O}} \ge n$. Consequently, the representation sphere $S^V \in \Sigma_{\ge n}^{\mathcal{O}}$.

We apply the isotropy separation sequence to $S^V$:
\[ E\mathcal{P}_+ \wedge S^V \to S^V \to \tilde{E}\mathcal{P} \wedge S^V. \]
Observe that for any $J \le G$, $\Phi^J(S^V) \simeq S^{V^J}$. By Lemma \ref{lem:dim_bound}, for any proper subgroup $J < G$, we have $\dim(V^J) \ge \dim(V) / \| J \|_{\mathcal{O}} \ge n / \| J \|_{\mathcal{O}}$. Thus, $S^V$ perfectly satisfies the geometric fixed point bounds for all proper subgroups. Applying the exact same induction logic established in Step 1 to the spectrum $Y = S^V$, we deduce that $i_H^* S^V \in \Sigma_{\ge n}^{\mathcal{O}_H}$ via the inductive hypothesis, ensuring $E\mathcal{P}_+ \wedge S^V \in \Sigma_{\ge n}^{\mathcal{O}}$. As the subcategory $\Sigma_{\ge n}^{\mathcal{O}}$ is closed under homotopy colimits (and thus cofibers), we deduce from the sequence that $\tilde{E}\mathcal{P} \wedge S^V \in \Sigma_{\ge n}^{\mathcal{O}}$.

Finally, we identify the $G$-fixed points $V^G$. Because $G$ acts transitively on the single orbit $G/G_{\mathcal{O}}$, the $G$-fixed point subspace (the trivial subrepresentation) of $\mathbb{R}[G/G_{\mathcal{O}}]$ is $1$-dimensional. Thus, $\dim(V^G) = k \cdot 1 = k$.

Passing to the $\tilde{E}\mathcal{P}$-localization evaluates exactly to these fixed points:
\[ \tilde{E}\mathcal{P} \wedge S^V \simeq \tilde{E}\mathcal{P} \wedge \operatorname{infl}(\Phi^G S^V) \simeq \tilde{E}\mathcal{P} \wedge S^{V^G} \simeq \tilde{E}\mathcal{P} \wedge S^k. \]
This formally establishes $\tilde{E}\mathcal{P} \wedge S^k \in \Sigma_{\ge n}^{\mathcal{O}}$. By generation, $\tilde{E}\mathcal{P} \wedge X \in \Sigma_{\ge n}^{\mathcal{O}}$, closing the induction loop.

\end{solution}
\setcitestyle{authoryear,round}

\subsection{Problem 7}

\begin{problem}{FirstProof \# 7}
Suppose that $\Gamma$ is a uniform lattice in a real semi-simple group, and that $\Gamma$ contains some 2-torsion. Is it possible for $\Gamma$ to be the fundamental group of a compact manifold without boundary whose universal cover is acyclic over the rational numbers $\mathbb{Q}$?
\end{problem}

\begin{solution}{Model Response (Agent B)}
\textbf{Theorem.} \textit{Let $G$ be a real semi-simple Lie group, and let $\Gamma \subset G$ be a uniform lattice containing an element of order $2$. Then $\Gamma$ cannot be the fundamental group of a compact manifold without boundary whose universal cover is acyclic over the rational numbers $\mathbb{Q}$.}

\begin{proof}
The proof proceeds by contradiction. Assume there exists a compact, boundaryless manifold $M$ with fundamental group $\pi_1(M) \cong \Gamma$ whose universal cover $\tilde{M}$ is rationally acyclic. Let $K$ be a maximal compact subgroup of $G$, and let $X = G/K$ be the associated symmetric space. As a symmetric space of non-compact type, $X$ is a contractible, proper $\mathrm{CAT}(0)$ metric space. Let $d = \dim X$. By hypothesis, $\Gamma$ contains an element $\gamma$ of order $2$.

\vspace{1em}
\textbf{Step 1: Topological Setup and Equivariant Classifying Map}

By Selberg's Lemma, the finitely generated linear group $\Gamma$ contains a torsion-free normal subgroup of finite index, say $\Gamma_1$. To ensure orientability and connectedness, we refine this subgroup. Let $G^0$ be the connected identity component of $G$. Let $\Gamma_X^+ \subset \Gamma$ be the subgroup acting by orientation-preserving isometries on $X$, and let $\Gamma_M^+ \subset \Gamma$ be the subgroup acting by orientation-preserving deck transformations on $\tilde{M}$. Because orientations admit exactly two states, these subgroups have index at most $2$ in $\Gamma$. We define
\[
\Gamma_0 = \Gamma_1 \cap G^0 \cap \Gamma_X^+ \cap \Gamma_M^+.
\]
As a finite intersection of finite-index normal subgroups, $\Gamma_0$ is a torsion-free normal subgroup of $\Gamma$ of finite index. Crucially, $\Gamma_0 \subset G^0$, and it acts freely and orientation-preservingly on both $\tilde{M}$ and $X$.

The quotient spaces $M_0 = \tilde{M}/\Gamma_0$ and $X_0 = X/\Gamma_0$ are closed, orientable manifolds. Because $X$ is contractible and the $\Gamma_0$-action is free and cocompact, $X_0$ natively serves as a compact Eilenberg--MacLane classifying space $B\Gamma_0$.

To rigorously construct a classifying map that is strictly pointwise equivariant, we apply Bredon equivariant obstruction theory on the universal covers. The full group $\Gamma$ acts strictly freely on $\tilde{M}$ via deck transformations, making $\tilde{M}$ a free $\Gamma$-CW complex. Furthermore, $X$ is a contractible proper $\Gamma$-space. The geometric obstructions to constructing a $\Gamma$-equivariant map reside in the Bredon equivariant cohomology groups $H^k_{\Gamma}(\tilde{M}; \pi_{k-1}(X))$. Since $X$ is contractible, $\pi_{k-1}(X) = 0$ for all $k \geq 1$, meaning all obstructions identically vanish. Thus, there exists a strictly $\Gamma$-equivariant continuous map $\tilde{f}: \tilde{M} \to X$. Descending this map to the $\Gamma_0$-quotients yields a canonical classifying map $f_0: M_0 \to X_0$, which is strictly $\Gamma/\Gamma_0$-equivariant by geometric construction.

Because both $\tilde{M}$ and $X$ are $\mathbb{Q}$-acyclic, the Cartan--Leray spectral sequence collapses, guaranteeing that $f_0$ induces an isomorphism on rational homology. Consequently, $\dim M_0 = \dim X_0 = d$, and its proper topological mapping degree $D = \deg(f_0)$ is a strictly non-zero integer.

Because $\Gamma_0$ is a normal subgroup, the order-$2$ element $\gamma \notin \Gamma_0$ projects to a non-trivial involution $\bar{\gamma} \in \Gamma/\Gamma_0$ acting on both $M_0$ and $X_0$. Let $\epsilon_M, \epsilon_X \in \{\pm 1\}$ denote the orientation parity of $\bar{\gamma}$ on $M_0$ and $X_0$, respectively. Because $f_0$ is strictly $\mathbb{Z}/2$-equivariant with respect to $\bar{\gamma}$ (i.e., $f_0 \circ \bar{\gamma} = \bar{\gamma} \circ f_0$), evaluating the induced homology maps on the fundamental class $[M_0] \in H_d(M_0; \mathbb{Z})$ yields:
\[
f_{0*}(\bar{\gamma}_*[M_0]) = f_{0*}(\epsilon_M [M_0]) = \epsilon_M D [X_0],
\]
\[
\bar{\gamma}_*(f_{0*}[M_0]) = \bar{\gamma}_*(D [X_0]) = \epsilon_X D [X_0].
\]
This algebraically mandates $\epsilon_M D = \epsilon_X D$. Since $D \neq 0$, we immediately obtain $\epsilon_M = \epsilon_X$. Thus, $\bar{\gamma}$ either preserves the orientation of both manifolds or reverses the orientation of both manifolds.

\vspace{1em}
\textbf{Step 2: The Mapping Degree Parity Constraint via Atiyah--Borel Localization}

We now establish that the proper mapping degree $D$ must be an \textit{even integer}.

Because $X$ is a complete $\mathrm{CAT}(0)$ metric space, Cartan's Fixed-Point Theorem ensures that the finite group $\langle \gamma \rangle$ fixes a point in $X$. This equivariance trivially descends to a fixed point for the involution $\bar{\gamma}$ on the quotient; thus, the fixed-point set $X_0^{\bar{\gamma}} \neq \emptyset$. Conversely, suppose $\bar{\gamma}$ fixed a point $[y] \in M_0$. The fixed-point relation would natively lift to $\gamma \tilde{y} = g_0 \tilde{y}$ for some $g_0 \in \Gamma_0$, where $\tilde{y} \in \tilde{M}$ represents a valid chosen lift of $[y]$. The freeness of the $\Gamma$-action on $\tilde{M}$ mandates $\gamma = g_0 \in \Gamma_0$. Since $\gamma$ has order $2$ and $\Gamma_0$ is torsion-free, this is mathematically impossible. Thus, $M_0^{\bar{\gamma}} = \emptyset$.

Assume for contradiction that $D$ is an odd integer. We evaluate $\mathbb{Z}/2$-equivariant Borel cohomology with $\mathbb{F}_2$ coefficients. To maintain orientability of the Borel constructions, we select the dimension $N$ of the approximating sphere $S^N$ based on the orientation parity $\epsilon_X$:
\begin{itemize}
    \item If $\epsilon_X = 1$ (orientation-preserving), we choose $N$ to be an odd integer, ensuring the antipodal map on $S^N$ preserves orientation.
    \item If $\epsilon_X = -1$ (orientation-reversing), we choose $N$ to be an even integer, ensuring the antipodal map on $S^N$ reverses orientation.
\end{itemize}
In both cases, the diagonal $\mathbb{Z}/2$-action on the products $M_0 \times S^N$ and $X_0 \times S^N$ strictly preserves orientation. Thus, the Borel quotients $M_N = M_0 \times_{\mathbb{Z}/2} S^N$ and $X_N = X_0 \times_{\mathbb{Z}/2} S^N$ are strictly closed, orientable manifolds. 

The equivariant map $f_0$ induces a proper fiber bundle map $f_N: M_N \to X_N$ of degree $D$. The ordinary cohomological Gysin transfer unconditionally satisfies $(f_N)_! \circ f_N^*(x) = (D \bmod 2) \cdot x$. Because $D$ is odd, $D \equiv 1 \pmod 2$, meaning the composition acts as the identity on $H^*(X_N; \mathbb{F}_2)$. This algebraically forces the pullback $f_N^*$ to be a split injection. 

Taking the inverse limit as $N \to \infty$ (over the parity-appropriate sequence of spheres), the true equivariant pullback $f_{\mathbb{Z}/2}^*: H_{\mathbb{Z}/2}^*(X_0; \mathbb{F}_2) \to H_{\mathbb{Z}/2}^*(M_0; \mathbb{F}_2)$ is forced to be injective. By the exactness of localization, passing to the localized rings by inverting the polynomial generator $u \in H^1(B\mathbb{Z}/2; \mathbb{F}_2)$ preserves this injectivity:
\[
S^{-1}f_{\mathbb{Z}/2}^*: S^{-1}H_{\mathbb{Z}/2}^*(X_0; \mathbb{F}_2) \hookrightarrow S^{-1}H_{\mathbb{Z}/2}^*(M_0; \mathbb{F}_2).
\]
By the Atiyah--Borel Localization Theorem, the localized equivariant cohomology of a space is isomorphic to the localized equivariant cohomology of its fixed-point set: 
\[
S^{-1}H_{\mathbb{Z}/2}^*(Y; \mathbb{F}_2) \cong H^*(Y^{\bar{\gamma}}; \mathbb{F}_2) \otimes_{\mathbb{F}_2} \mathbb{F}_2[u, u^{-1}].
\]
Because $X_0^{\bar{\gamma}} \neq \emptyset$, the localized domain evaluates to a non-zero module. Because $M_0^{\bar{\gamma}} = \emptyset$, the localized target is identically zero. An exact injection from a non-zero module into a zero module is impossible. Therefore, $D$ must rigidly be an \textbf{even integer}.

\vspace{1em}
\textbf{Step 3: Higher Signature Rigidity via the Strong Novikov Conjecture}

We evaluate the symmetric higher signature of $M_0$ within the real $C^*$-algebra K-theory $KO_d(C^*_r(\Gamma_0; \mathbb{R}))$ to establish a fundamentally conflicting constraint.

Because $\tilde{M}$ and $X$ are $\mathbb{Q}$-acyclic, the Universal Coefficient Theorem guarantees they are $\mathbb{R}$-acyclic. Thus, their real cellular chain complexes $C_*(\tilde{M}; \mathbb{R})$ and $C_*(X; \mathbb{R})$ constitute finite-length projective free resolutions of the trivial module $\mathbb{R}$ over the real group ring $\mathbb{R}\Gamma_0$. By the Fundamental Lemma of Homological Algebra, the chain map $C_*(\tilde{f}; \mathbb{R})$ is inherently a strict chain homotopy equivalence over $\mathbb{R}\Gamma_0$.

The algebraic symmetric Poincar\'e duality structure of $M_0$ (defined by capping with $[M_0]$) pushes forward under this chain equivalence to capping with $f_{0*}[M_0] = D[X_0]$. Thus, the algebraic symmetric Poincar\'e complex defining the higher signature $\sigma(M_0)$ is strictly chain-equivalent to the complex of $X_0$ globally scaled by $D$.

Over any real $C^*$-algebra, scaling a non-degenerate symmetric Poincar\'e complex by a non-zero real scalar $D$ yields a complex isomorphic to one scaled merely by its sign. Specifically, we can algebraically conjugate the duality structure with the central, self-adjoint, invertible scalar $c = 1/\sqrt{|D|} \in \mathbb{R}$. This canonical isomorphism scales the symmetric form by exactly $c \cdot c^* \cdot D = c^2 \cdot D = D/|D| = \operatorname{sgn}(D)$. Therefore, the analytic symmetric signatures natively satisfy:
\[
\sigma(M_0) = \operatorname{sgn}(D) \sigma(X_0) \in KO_d(C^*_r(\Gamma_0; \mathbb{R})).
\]

Because $\Gamma_0 \subset G^0$ is a discrete uniform lattice in a connected Lie group, Kasparov's foundational theorem (G. Kasparov, ``Equivariant KK-theory and the Novikov conjecture'', \textit{Inventiones Mathematicae} 91 (1988), 147--201) establishing the Strong Novikov Conjecture for connected Lie groups guarantees that the rational assembly map
\[
\mu: \bigoplus_k H_{d-4k}(B\Gamma_0; \mathbb{Q}) \to KO_d(C^*_r(\Gamma_0; \mathbb{R})) \otimes \mathbb{Q}
\]
is injective. 

The analytic symmetric signature fundamentally corresponds to the assembly of the Poincar\'e dual of the Hirzebruch $L$-class. Explicitly incorporating the functorial pushforward of the homology class via the classifying map $f_0$, we have:
\[
\sigma(M_0) = \mu(f_{0*}(L(M_0) \cap [M_0])) \quad \text{and} \quad \sigma(X_0) = \mu(L(X_0) \cap [X_0]).
\]
Because the assembly map $\mu$ is rationally injective, the strict equality of K-theoretic signatures forces strict equality of the underlying rationally assembled homology classes in $H_*(X_0; \mathbb{Q})$:
\[
f_{0*} (L(M_0) \cap [M_0]) = \operatorname{sgn}(D) (L(X_0) \cap [X_0]).
\]

We evaluate this rigid homological equality strictly at the top topological degree $d$. Because the degree-$0$ component of the Hirzebruch $L$-class is identically $L_0 = 1 \in H^0(Y; \mathbb{Q})$, the top-dimensional component of the mixed-degree class $L(Y) \cap [Y]$ evaluates exactly to the fundamental class $[Y] \in H_d(Y; \mathbb{Q})$. Restricting the graded equality exclusively to $H_d(X_0; \mathbb{Q})$ yields:
\[
f_{0*}[M_0] = \operatorname{sgn}(D) [X_0].
\]

However, by the geometric definition of the proper mapping degree, $f_{0*}[M_0] = D [X_0]$. Equating the two algebraic coefficients rigorously enforces:
\[
D = \operatorname{sgn}(D).
\]
Because $D$ is a strictly non-zero integer, this analytic rigidity relation exclusively bounds the proper mapping degree to exactly \textbf{$D = 1$ or $D = -1$}. 

\vspace{1em}
\textbf{Step 4: Final Contradiction}

In Step 2, the generalized Atiyah--Borel equivariant parity constraint definitively established that the proper mapping degree $D$ must be an \textit{even integer} to structurally sustain the free $\mathbb{Z}/2$-action without fixed points. 

Simultaneously, in Step 3, analytic higher signature rigidity (via the Strong Novikov Conjecture) established that $D$ must equal precisely $\pm 1$ to preserve the acyclic chain equivalence over the real group $C^*$-algebra. 

An integer cannot be simultaneously strictly even and equal to $\pm 1$. This terminal geometric and algebraic contradiction mathematically proves that our initial assumption was fundamentally false. Therefore, it is definitively impossible for such a manifold $M$ to exist.
\end{proof}

\end{solution}

\begin{remark}
The ``Atiyah--Borel Localization Theorem'' described by the model appears as Theorem 2.1 of \citep{DW88}, and should be credited differently. 
\end{remark}

\subsection{Problem 8}

\begin{problem}{FirstProof \# 8}A polyhedral Lagrangian surface $K$ in $\bR^4$ is a finite polyhedral complex all of whose faces are Lagrangians, and which is a topological submanifold of $\bR^4$. A Lagrangian smoothing of $K$ is a Hamiltonian isotopy $K_t$ of smooth Lagrangian submanifolds, parameterised by $(0,1]$, extending to a topological isotopy, parametrised by $[0,1]$, with endpoint $K_0 = K$.

Let $K$ be a polyhedral Lagrangian surface with the property that exactly $4$ faces meet at every vertex. Does $K$ necessarily have a Lagrangian smoothing? 
\end{problem}

\begin{remark}
An expert wrote the following (paraphrased) summary of the argument. 

\begin{quote}The general strategy is as follows. In Steps 1 and 2, the solution first performs a local smoothing near the vertices, relying on a nice linear algebra observation. In Step 3, the solution then argues that one can extend this smoothing along the edges. To do this, it makes a choice of linear coordinates $(s,y, x_2, y_2)$ near each edge. It then interpolates between the perturbations in this coordinate chart. To perform the interpolation, the agent invokes Lagrangian suspension, which is overkill but fine. In Step 4, the solution is presumably saying that as you send the size of vertex-smoothings to zero, you can correspondingly send the size of the edge perturbations to zero. The explanation for the flux being zero is sketchy to say the least. 
\end{quote}
\end{remark}

\begin{solution}{Model Response (Agent B)}

\textbf{Theorem.} \textit{Let $K \subset \mathbb{R}^4$ be a polyhedral Lagrangian surface such that exactly four faces meet at every vertex. Then $K$ necessarily admits a Lagrangian smoothing.}

\begin{proof}
The proof relies on the rigid symplectic geometry imposed by the $4$-valent condition. We demonstrate that the tangent cone at any non-flat vertex canonically factors into an orthogonal product involving lower-dimensional corners. This algebraic factorization enables the construction of strictly exact (zero-flux) local smoothings at the vertices. We then globally assemble these smoothings by interpolating them along the singular edges using exact Lagrangian suspensions (cobordisms). Because all local modifications are engineered to enclose strictly zero relative symplectic area, the resulting globally smooth surface accumulates identically zero symplectic flux, which rigorously guarantees its extension to a global ambient Hamiltonian isotopy without any topological obstructions.

\textbf{Step 1: Canonical Symplectic Splitting at the Vertices}

Let $v \in K$ be an arbitrary vertex. Without loss of generality, translate $v$ to the origin. Since exactly four faces meet at $v$ and $K$ is an embedded topological surface, the local tangent cone $C_v$ consists of four $2$-dimensional planar sectors meeting at the origin in a continuous cycle. Let the outgoing boundary rays of the edges be generated by non-zero tangent vectors $r_1, r_2, r_3, r_4$ in cyclic order. The bounding rays are thus $R_i = \mathbb{R}_{\ge 0} r_i$, and the faces are modeled by the sectors $F_i = \mathrm{span}_{\ge 0}(r_i, r_{i+1})$ for $i \in \{1, 2, 3, 4\}$ (indices modulo $4$).

Because $K$ is a Lagrangian complex, the standard symplectic form $\omega$ on $\mathbb{R}^4$ vanishes identically on each sector $F_i$. This implies that adjacent boundary tangent vectors are mutually $\omega$-orthogonal:
\[ \omega(r_1, r_2) = \omega(r_2, r_3) = \omega(r_3, r_4) = \omega(r_4, r_1) = 0. \]
Let $V = \mathrm{span}(r_1, r_2, r_3, r_4)$ be the vector space spanned by the tangent cone. We classify the local geometry of $C_v$ based on the dimension of $V$:

\textit{Case 1: $\dim V = 4$ (Strict Vertex).}
Define the $2$-dimensional planes $P_{13} = \mathrm{span}(r_1, r_3)$ and $P_{24} = \mathrm{span}(r_2, r_4)$. The plane $P_{13}$ cannot be isotropic; if it were, $r_1$ and $r_3$ would be mutually $\omega$-orthogonal. Combined with the incidence orthogonality inherited from the faces, $r_1$ would be $\omega$-orthogonal to $r_1, r_2, r_3$, and $r_4$. Since these vectors span all of $V = \mathbb{R}^4$, $r_1$ would be $\omega$-orthogonal to the entirety of $\mathbb{R}^4$. By the non-degeneracy of $\omega$, this forces $r_1 = 0$, a contradiction. Thus, $P_{13}$ is a strictly symplectic $2$-plane. 

By the incidence relations, every vector in $P_{24}$ is $\omega$-orthogonal to every vector in $P_{13}$, meaning $P_{24} \subseteq P_{13}^\omega$. Since $P_{13}$ is a symplectic plane, its symplectic orthogonal complement $P_{13}^\omega$ is also a $2$-dimensional symplectic plane. Because $\dim P_{24} = 2$ (if the generating vectors were collinear, $\dim V$ would drop to $\le 3$), it follows identically that $P_{24} = P_{13}^\omega$. This yields an orthogonal symplectic direct sum $\mathbb{R}^4 = P_{13} \oplus P_{24}$. Geometrically, the tangent cone strictly factors into a Cartesian product of two $1$-dimensional corners: 
\[ C_v = C_{13} \times C_{24} \subset P_{13} \oplus P_{24}, \quad \text{where } C_{13} = R_1 \cup R_3 \text{ and } C_{24} = R_2 \cup R_4. \]

\textit{Case 2: $\dim V = 3$ (Crease Vertex).}
The restriction $\omega|_V$ on the $3$-dimensional space $V$ has rank $2$ and must therefore possess an exactly $1$-dimensional radical $L$. The four adjacent plane spans $S_i = \mathrm{span}(r_i, r_{i+1})$ are maximal isotropic subspaces within the presymplectic space $V$. Because $L$ is the radical, any maximal isotropic subspace must contain $L$; thus, $L \subset S_i$ for all $i$.

Since $\dim V = 3$, the adjacent plane spans cannot all be equal. By cyclic symmetry, we may assume without loss of generality that $S_1 \neq S_2$. Since both are $2$-dimensional planes in a $3$-dimensional space, their intersection is exactly $1$-dimensional. Because $L \subset S_1$ and $L \subset S_2$, this intersection must be exactly $L$. However, the shared boundary tangent vector $r_2$ lies in $S_1 \cap S_2$, which strictly forces $L = \mathrm{span}(r_2)$. Because the sector $F_2$ is a valid, non-degenerate $2$-dimensional cone, its boundary vectors $r_2$ and $r_3$ are linearly independent. Thus, $r_3$ cannot span $L$. This immediately implies that $S_2$ and $S_3$ cannot be distinct (otherwise $L = \mathrm{span}(r_3)$ by identical logic). Thus $S_2 = S_3$. 

Similarly, $r_1$ cannot span $L$, strictly forcing $S_4 = S_1$. Therefore, the plane spans coincide in adjacent pairs. Since $S_1 \neq S_3$ (otherwise all generating vectors would be coplanar and $\dim V = 2$), their single intersection $S_1 \cap S_3$ contains both $r_2$ and $r_4$, yielding exactly $L = \mathrm{span}(r_2) = \mathrm{span}(r_4)$. Because the rays $R_2$ and $R_4$ bound non-overlapping, valid topological sectors, they must be opposite rays ($r_4 = -c r_2$ for some $c>0$) spanning the singular line $L$. The adjacent sectors merge into two flat half-planes meeting along $L$. Geometrically, the tangent cone $C_v$ factors into a Cartesian product $L \times C^\pitchfork$, where $C^\pitchfork$ is a $1$-dimensional corner in the $2$-dimensional symplectic quotient space $V/L$.

\textit{Case 3: $\dim V = 2$ (Flat Vertex).}
If $\dim V = 2$, $V$ is a $2$-dimensional Lagrangian plane (since it is spanned by isotropic sectors). The standard symplectic form $\omega$ vanishes identically on $V$. The four generating tangent vectors lie in $V$ in cyclic order. Because exactly four faces meet at $v$ and $K$ forms a topological surface, the four convex sectors perfectly tile a neighborhood of the origin in $V$ without any gaps or overlaps. Therefore, the tangent cone $C_v$ is exactly the completely flat plane $V$ itself, meaning the vertex is inherently smooth and requires no local modification.

\textbf{Step 2: Exact Local Smoothing of the Vertices}

We define an \textit{exact} smooth local modification $\Sigma_v$ for each type of vertex $v$:

\textit{Strict vertex ($\dim V = 4$):}
We resolve the corners $C_{13} \subset P_{13}$ and $C_{24} \subset P_{24}$ independently. In $P_{13}$, we select a smooth, embedded $1$-dimensional curve $\gamma_{13}$ that rounds the corner $C_{13}$ and strictly coincides with the rays $R_1, R_3$ outside a compact ball of radius $R$. Crucially, to ensure that the local vertex modifications do not overlap along the edges, we explicitly require $R < \frac{1}{2} \min_{E} L_E$, where the minimum is taken over all edge lengths $L_E$ in $K$. We require this smoothing to be \textit{exact}: the signed symplectic area enclosed between $\gamma_{13}$ and $C_{13}$ is identically zero (achieved by allowing $\gamma_{13}$ to smoothly dip slightly outside the sector's bounds to balance the removed positive area). We symmetrically choose an exact smoothing $\gamma_{24} \subset P_{24}$ under the identical radius bound $R$. Because $P_{13}$ and $P_{24}$ are symplectically orthogonal, their Cartesian product $\Sigma_v = \gamma_{13} \times \gamma_{24}$ is a smooth, exact Lagrangian surface that locally resolves $C_v$.

\textit{Crease vertex ($\dim V = 3$):}
The tangent cone is $C_v = L \times C^\pitchfork$. We choose a smooth, exact $1$-dimensional curve $\gamma^\pitchfork \subset V/L$ that rounds the corner $C^\pitchfork$, subject to the strict upper bound on the modification radius $R$. We define the smoothing as $\Sigma_v = L \times \gamma^\pitchfork$. Because $L$ is the radical of $\omega|_V$, $\Sigma_v$ is an isotropic surface; being $2$-dimensional, it is a smooth, exact Lagrangian plane.

\textit{Flat vertex ($\dim V = 2$):}
Because $C_v = V$ is a smooth plane, we trivially set $\Sigma_v = V$, which is inherently exact.

\textbf{Step 3: Edge Interpolation via Lagrangian Suspension}

We now interpolate the exact local vertex smoothings along the edges of $K$. If the two faces meeting at an edge $E$ are coplanar, the surface is a locally flat plane along $E$ and requires no interpolation. We therefore restrict attention to singular edges $E$ of length $L_E$ connecting vertices $v_0$ and $v_1$. The $2$-dimensional linear spans of the two non-coplanar flat faces meeting at $E$, denoted $\mathrm{span}(F_L)$ and $\mathrm{span}(F_R)$, define a constant $3$-dimensional coisotropic subspace $Y_E = \mathrm{span}(F_L) + \mathrm{span}(F_R)$.

Because $\mathrm{span}(F_L)$ and $\mathrm{span}(F_R)$ are Lagrangian planes, their symplectic orthogonals satisfy $\mathrm{span}(F_L)^\omega = \mathrm{span}(F_L)$ and $\mathrm{span}(F_R)^\omega = \mathrm{span}(F_R)$. Consequently, the symplectic orthogonal complement of $Y_E$ is exactly $Y_E^\omega = (\mathrm{span}(F_L) + \mathrm{span}(F_R))^\omega = \mathrm{span}(F_L) \cap \mathrm{span}(F_R) = \mathrm{span}(E)$. The symplectic quotient $W_E = Y_E / \mathrm{span}(E)$ is a $2$-dimensional symplectic plane. The geometric projection of the subsets $F_L \cup F_R$ into $W_E$ forms a fixed $1$-dimensional corner $C_E$.

Outside the immediate vertex neighborhoods, the local exact smoothings $\Sigma_{v_0}$ and $\Sigma_{v_1}$ seamlessly restrict along $E$ to products over transverse curves $\Gamma_0, \Gamma_1 \subset W_E$ that smooth $C_E$. Because the local models were constructed to be exact, both $\Gamma_0$ and $\Gamma_1$ bound identically zero symplectic area with $C_E$, and thus zero algebraic area with each other. By the area-preserving mapping theorem (Moser's trick) on the plane $W_E$, there exists a compactly supported, time-dependent Hamiltonian $H_s : W_E \to \mathbb{R}$ for $s \in [0, L_E]$ whose exact flow $\Phi_s$ smoothly isotopes $\Gamma_0$ to $\Gamma_1$ (such that $\Phi_{L_E}(\Gamma_0) = \Gamma_1$), with $H_s \equiv 0$ in small neighborhoods of the endpoints $s=0$ and $s=L_E$.

We construct the interpolation surface $\Sigma_E$ along the edge via an exact Lagrangian suspension. Because $E$ is a straight segment, we can establish global linear Darboux coordinates $(s, y, x_2, y_2)$ adapted to $E$ such that $s \in [0, L_E]$ parameterizes the edge $E$, $(x_2, y_2)$ are canonical Darboux coordinates for the symplectic slice $W_E$, and $y$ is the conjugate normal momentum. Specifically, the coordinate vector field $\partial_y$ is strictly $\omega$-orthogonal to $W_E$ and normalized so that $\omega(\partial_s, \partial_y) = 1$. The unperturbed coisotropic subspace $Y_E$ corresponds precisely to the hyperplane $\{y=0\}$.

In these coordinates, the ambient symplectic form evaluates to $\omega = ds \wedge dy + \omega_{W_E}$. We define the suspended surface dynamically:
\[ \Sigma_E = \Big\{ \Big(s, \,\, -H_s(\Phi_s(q)), \,\, \Phi_s(q) \Big) \;\Big|\; s \in [0, L_E], \; q \in \Gamma_0 \Big\}. \]
To verify that $\Sigma_E$ is Lagrangian, we pull back the symplectic form via the parameterization map $F(s, q) = (s, -H_s(\Phi_s(q)), \Phi_s(q))$. The differential of the $y$-coordinate yields $dy = -d_q(H_s \circ \Phi_s) - \frac{\partial (H_s \circ \Phi_s)}{\partial s} ds$. Wedging with $ds$ eliminates the purely temporal term:
\[ F^*(ds \wedge dy) = -ds \wedge d_q(H_s \circ \Phi_s). \]
Evaluating the pullback of $\omega_{W_E}$ on tangent vectors $\partial_s$ and $v \in T_q \Gamma_0$, we apply the defining relation of the Hamiltonian vector field $\iota_{X_{H_s}} \omega_{W_E} = dH_s$:
% \[ (F^* \omega_{W_E})(\partial_s, v) = \omega_{W_E}(\partial_s \Phi_s, d_q \Phi_s(v)) = \omega_{W_E}((X_{H_s})_{\Phi_s(q)}, d_q \Phi_s(v)) = (dH_s)_{\Phi_s(q)}(d_q \Phi_s(v)) = d_q(H_s \circ \Phi_s)(v). \]
\begin{align*}
(F^* \omega_{W_E})(\partial_s, v) &= \omega_{W_E}(\partial_s \Phi_s, d_q \Phi_s(v)) \\
&= \omega_{W_E}((X_{H_s})_{\Phi_s(q)}, d_q \Phi_s(v)) \\
&= (dH_s)_{\Phi_s(q)}(d_q \Phi_s(v)) \\
&= d_q(H_s \circ \Phi_s)(v)
\end{align*}
Because $\Gamma_0$ is a $1$-dimensional curve, the restriction of $\omega_{W_E}$ to $\Gamma_0$ evaluates to identically zero. The full pullback is thus exactly $F^*\omega_{W_E} = ds \wedge d_q(H_s \circ \Phi_s)$. Summing these contributions yields perfect cancellation via the chain rule:
\[ F^*\omega = -ds \wedge d_q(H_s \circ \Phi_s) + ds \wedge d_q(H_s \circ \Phi_s) = 0. \]
Thus, $\Sigma_E$ is strictly Lagrangian. Furthermore, because $H_s \equiv 0$ near $s=0$ and $s=L_E$, the normal displacement $y = -H_s(\Phi_s(q))$ vanishes identically there. This ensures that $\Sigma_E$ resides strictly inside $Y_E$ at its endpoints, flawlessly merging with the exact local planar bounds of the vertex models $\Sigma_{v_0}$ and $\Sigma_{v_1}$.

\textbf{Step 4: Global Hamiltonian Isotopy}

By uniformly patching the exact vertex smoothings $\Sigma_v$ and the exact edge suspensions $\Sigma_E$, we obtain a globally smooth, embedded Lagrangian surface $K_1$. 

To construct the required continuous topological isotopy $K_t$ for $t \in (0, 1]$ extending continuously to $K_0 = K$, we linearly rescale the exact local modifications. The vertex smoothings uniformly shrink toward the original corners via the spatial scaling $t\Sigma_v$. For the edge interpolations, the scaled transverse curves $t\Gamma_0$ and $t\Gamma_1$ are isotoped by the rescaled Hamiltonian $H_s^{(t)}(q) = t^2 H_s(q/t)$, generating the exact scaled flow $\Phi_s^{(t)}(q) = t \Phi_s(q/t)$. The normal Darboux displacement within the suspension is given by $y = -H_s^{(t)} \circ \Phi_s^{(t)}$, which scales strictly as $\mathcal{O}(t^2)$.

As $t \to 0$, the transverse smoothing limits uniformly scale as $\mathcal{O}(t)$ while the normal $y$-displacements scale as $\mathcal{O}(t^2)$, ensuring that both strictly vanish in the limit. The continuous family of smooth Lagrangian submanifolds $K_t$ ($t>0$) thereby strictly converges in $C^0$ to the initial topological complex $K_0 = K$.

Finally, because every local geometric modification—vertex roundings and edge suspensions—was rigorously constructed as an exact Lagrangian cobordism relative to the unperturbed flat faces, the global symplectic flux is strictly zero. By Weinstein's Lagrangian Neighborhood Theorem, this exact regular homotopy of smooth Lagrangian embeddings $K_t$ lifts to a compactly supported, global ambient Hamiltonian isotopy. Therefore, $K_t$ constitutes a genuine Lagrangian smoothing of $K$.
\end{proof}
\end{solution}

\subsection{Problem 9}
\begin{problem}{FirstProof \# 9}Let $n \geq 5$.  
Let $A^{(1)}, \ldots, A^{(n)} \in \mathbb{R}^{3 \times 4}$ be Zariski-generic.   
For $\alpha, \beta, \gamma, \delta \in [n]$, construct $Q^{(\alpha \beta \gamma \delta)} \in \mathbb{R}^{3 \times 3 \times 3 \times 3}$ so that its $(i, j, k, \ell)$ entry for $1 \leq i, j, k, \ell \leq 3$ is given by $Q^{(\alpha \beta \gamma \delta)}_{i j k \ell} = \det [A^{(\alpha)}(i, :); A^{(\beta)}(j, :); A^{(\gamma)}(k, :); A^{(\delta)}(\ell, :)]$.
Here $A(i, :)$ denotes the $i$th row of a matrix $A$, and semicolon denotes vertical concatenation. 
We are interested in algebraic relations on the set of tensors $\{Q^{(\alpha \beta \gamma \delta)} : \alpha, \beta, \gamma, \delta \in [n] \}$.

More precisely, does there exist a polynomial map $\mathbf{F}: \mathbb{R}^{81n^4} \rightarrow \mathbb{R}^N$ that satisfies the following three properties?
\smallskip
\begin{itemize}\setlength\itemsep{0.5em}
\item The map $\mathbf{F}$ does not depend on $A^{(1)}, \ldots A^{(n)}$. 
\item The degrees of the coordinate functions of $\mathbf{F}$ do not depend on $n$.
\item Let $\lambda \in \mathbb{R}^{n \times n \times n \times n}$ satisfy 
$\lambda_{\alpha \beta \gamma \delta} \neq 0$ for precisely $\alpha, \beta, \gamma, \delta \in [n]$ that are not identical.  Then $\mathbf{F}(\lambda_{\alpha \beta \gamma \delta} Q^{(\alpha \beta \gamma \delta)} : \alpha, \beta, \gamma, \delta \in [n]) = 0$ holds if and only if there exist $u, v, w, x \in (\mathbb{R}^*)^n$ such that $\lambda_{\alpha \beta \gamma \delta} = u_{\alpha} v_{\beta} w_{\gamma} x_{\delta}$ for all $\alpha, \beta, \gamma, \delta \in [n]$ that are not identical. 
\end{itemize}
\end{problem}

\begin{solution}{Model Response (Agent A$\mathfrak{f}$)}

\section*{Construction of the Polynomial Map $\mathbf{F}$}

Let $A^{(1)}, \ldots, A^{(n)} \in \mathbb{R}^{3 \times 4}$ be Zariski-generic matrices. We identify the formal input space $\mathbb{R}^{81n^4}$ with the vector space of 4-way tensors $\mathbf{X} \in \mathbb{R}^{3n \times 3n \times 3n \times 3n}$. We index the components of this tensor via composite multi-indices $I = (\alpha, i)$, $J = (\beta, j)$, $K = (\gamma, k)$, and $L = (\delta, \ell)$, where $\alpha, \beta, \gamma, \delta \in [n]$ specify the choice of matrices and $i, j, k, \ell \in \{1, 2, 3\}$ specify their localized row indices. Evaluated over the parameters of the problem, the tensor coordinates map to:
\[
\mathbf{X}_{I J K L} = \lambda_{\alpha \beta \gamma \delta} Q^{(\alpha \beta \gamma \delta)}_{i j k \ell} = \lambda_{\alpha \beta \gamma \delta} \det \begin{bmatrix} A^{(\alpha)}(i, :) \\ A^{(\beta)}(j, :) \\ A^{(\gamma)}(k, :) \\ A^{(\delta)}(\ell, :) \end{bmatrix}.
\]

\begin{remark} \label{rem:identical_tuples}
We say a parameter tuple $(\alpha, \beta, \gamma, \delta)$ is \emph{monochromatic} if $\alpha = \beta = \gamma = \delta$, and \emph{non-monochromatic} otherwise. For any $\alpha \in [n]$, the array $Q^{(\alpha \alpha \alpha \alpha)}$ evaluates the determinant of a $4 \times 4$ matrix whose four rows are drawn entirely from the same $3 \times 4$ matrix $A^{(\alpha)}$. By the Pigeonhole Principle, at least two of these row choices must evaluate to the identical row of $A^{(\alpha)}$, making the determinant unconditionally vanish. Thus, $Q^{(\alpha \alpha \alpha \alpha)} = 0$ universally. 

Consequently, $\mathbf{X}_{I J K L} = 0$ unconditionally for all monochromatic tuples. The problem dictates that $\lambda_{\alpha \beta \gamma \delta} \neq 0$ solely across non-monochromatic configurations; nevertheless, extending the scalar parameters $\lambda_{\alpha \alpha \alpha \alpha}$ arbitrarily over the monochromatic bounds leaves the evaluated tensor $\mathbf{X}$ completely unaltered.
\end{remark}

We define four principal multilinear matrix flattenings of $\mathbf{X}$, each mapping naturally to a structured matrix of dimensions $3n \times 27n^3$:
\begin{itemize}\setlength\itemsep{0.2em}
    \item $M^{(1)}$: Rows indexed by $I$, columns by $C_1 = (J, K, L)$.
    \item $M^{(2)}$: Rows indexed by $J$, columns by $C_2 = (I, K, L)$.
    \item $M^{(3)}$: Rows indexed by $K$, columns by $C_3 = (I, J, L)$.
    \item $M^{(4)}$: Rows indexed by $L$, columns by $C_4 = (I, J, K)$.
\end{itemize}

\begin{definition}
We define the polynomial map $\mathbf{F}: \mathbb{R}^{81n^4} \to \mathbb{R}^N$, where $N = 4 \binom{3n}{5} \binom{27n^3}{5}$, such that its coordinate functions evaluate all $5 \times 5$ minors across the four flattenings $M^{(1)}, M^{(2)}, M^{(3)}$, and $M^{(4)}$.
\end{definition}

This multilinear representation immediately secures the problem's first two requisite properties:
\begin{itemize}
    \item \textbf{Property 1:} The coordinate functions of $\mathbf{F}$ are standard determinantal minor expansions evaluated strictly over the formal tensor variables $\mathbf{X}_{I J K L}$. Their coefficients consist exclusively of the constants $\pm 1$ and $0$. Thus, the polynomial map $\mathbf{F}$ operates completely independently of the underlying generic matrices $A^{(1)}, \ldots, A^{(n)}$.
    \item \textbf{Property 2:} Each coordinate function extracts a $5 \times 5$ minor, rigorously defining it as a homogeneous polynomial of exact degree $5$ over the tensor inputs. This uniform degree is invariant and strictly independent of $n$.
\end{itemize}

\section*{Proof of Property 3: Sufficiency}

Assume there exist scalar vectors $u, v, w, x \in (\mathbb{R}^*)^n$ such that $\lambda_{\alpha \beta \gamma \delta} = u_\alpha v_\beta w_\gamma x_\delta$ holds across all valid non-monochromatic configurations. By Remark \ref{rem:identical_tuples}, since $Q^{(\alpha \alpha \alpha \alpha)} = 0$, applying the identically factored substitution $\lambda_{\alpha \alpha \alpha \alpha} = u_\alpha v_\alpha w_\alpha x_\alpha$ over the excluded monochromatic bounds leaves $\mathbf{X}$ perfectly unaltered. Absorbing these parameters via the multilinearity of the determinant globally yields:
\[
\mathbf{X}_{I J K L} = \det \begin{bmatrix} u_\alpha A^{(\alpha)}(i, :) \\ v_\beta A^{(\beta)}(j, :) \\ w_\gamma A^{(\gamma)}(k, :) \\ x_\delta A^{(\delta)}(\ell, :) \end{bmatrix}.
\]
For the first flattening $M^{(1)}$, let the localized row vector $U_{I} = u_\alpha A^{(\alpha)}(i, :) \in \mathbb{R}^4$. Expanding the determinant via Laplace expansion along this leading row extracts:
\[
M^{(1)}_{I, C_1} = \sum_{m=1}^4 (U_{I})_m \cdot \operatorname{cofactor}_{1,m} \begin{bmatrix} U_I \\ v_\beta A^{(\beta)}(j, :) \\ w_\gamma A^{(\gamma)}(k, :) \\ x_\delta A^{(\delta)}(\ell, :) \end{bmatrix}.
\]
The four scalar cofactor terms intrinsically evaluate using exclusively the column configuration $C_1$ and remain completely decoupled from the localized row index $I$. Hence, $M^{(1)}$ structurally factors into the matrix product of a $3n \times 4$ matrix and a $4 \times 27n^3$ matrix. This mathematically guarantees $\operatorname{rank}(M^{(1)}) \le 4$, geometrically forcing all of its $5 \times 5$ minors to evaluate to zero. Symmetric parity across the exterior maps subsequently ensures $\operatorname{rank}(M^{(m)}) \le 4$ for all flattenings $m \in \{1,2,3,4\}$, unconditionally verifying $\mathbf{F}(\mathbf{X}) = \mathbf{0}$.

\section*{Proof of Property 3: Necessity}

Assume $\mathbf{F}(\mathbf{X}) = \mathbf{0}$. The universal vanishing of all $5 \times 5$ minors strictly bounds the rank identically as $\operatorname{rank}(M^{(m)}) \le 4$ across all four principal flattenings.

\subsection*{Subspace Intersections and the Evaluation Map}

Let $S \subset \mathbb{R}^{27n^3}$ be the row space of $M^{(1)}$, which inherently satisfies $\dim S \le 4$. Let $U_\alpha = \operatorname{rowspan}(A^{(\alpha)}) \subset \mathbb{R}^4$ denote the generic 3-dimensional row space of matrix $A^{(\alpha)}$. We define a linear evaluation map $T_\alpha : U_\alpha \to S$ that maps a generic spatial vector $y = \sum_{i=1}^3 c_i A^{(\alpha)}(i, :) \in U_\alpha$ into the equivalent linear combination of the corresponding rows within $S$. Evaluated locally on a subset of columns forming a fixed block $B = (\beta, \gamma, \delta) \in [n]^3$, this equivalently leverages multilinearity to output:
\[
T_\alpha(y)_B = \lambda_{\alpha B} \Psi_B(y), \quad \text{where} \quad \Psi_B(y)_{j k \ell} = \det \begin{bmatrix} y \\ A^{(\beta)}(j, :) \\ A^{(\gamma)}(k, :) \\ A^{(\delta)}(\ell, :) \end{bmatrix},
\]
and $\lambda_{\alpha B}$ abbreviates $\lambda_{\alpha \beta \gamma \delta}$. Evaluating $\Psi_B(y) = 0$ is algebraically equivalent to stating that $y \wedge w_1 \wedge w_2 \wedge w_3 = 0$ within the exterior algebra $\Lambda^4 \mathbb{R}^4$ for all valid combinations $w_1 \in U_\beta, w_2 \in U_\gamma, w_3 \in U_\delta$.

\begin{lemma} \label{lem:kernel}
Let $V = \mathbb{R}^4$, and let $A^{(1)}, \ldots, A^{(n)}$ be generic $3 \times 4$ matrices with row spaces $U_i = \operatorname{rowspan}(A^{(i)})$. 
\begin{enumerate}
    \item[(i)] If $B = (\beta, \gamma, \delta)$ is non-monochromatic, then $\ker \Psi_B = \{0\}$.
    \item[(ii)] If $B = (\beta, \beta, \beta)$ is monochromatic, then $\ker \Psi_B = U_\beta$.
\end{enumerate}
\end{lemma}
\begin{proof}
The constraint $\Psi_B(y) = 0$ requires $y \wedge w_1 \wedge w_2 \wedge w_3 = 0$ for all $w_1 \in U_\beta, w_2 \in U_\gamma, w_3 \in U_\delta$. 

(i) Assume $B$ is non-monochromatic. Since the wedge product is commutative up to sign, we may assume without loss of generality that $\beta \neq \delta$. We consider the structural span of the 2-forms $w_1 \wedge w_2$.
If $\beta \neq \gamma$, $U_\beta$ and $U_\gamma$ are distinct generic 3-dimensional subspaces intersecting in a 2-dimensional subspace within $V$. Constructing a basis adapted to this intersection yields 6 linearly independent 2-forms, proving the span of $w_1 \wedge w_2$ covers the entirety of $\Lambda^2 V$.
If $\beta = \gamma$, the span of $w_1 \wedge w_2$ for $w_1, w_2 \in U_\beta$ evaluates exactly to $\Lambda^2 U_\beta$, a 3-dimensional subspace natively housed within $\Lambda^2 V$.

In both cases, the span contains $\Lambda^2 U_\beta$. Consequently, the overarching span of $w_1 \wedge w_2 \wedge w_3$ contains $\Lambda^2 U_\beta \wedge U_\delta$. 
Since $\beta \neq \delta$, the generic 3-dimensional subspaces $U_\beta$ and $U_\delta$ reliably intersect in a 2-dimensional subspace. By decomposing this space as $U_\delta = (U_\beta \cap U_\delta) \oplus \operatorname{span}(v)$ for a specific $v \in U_\delta \setminus U_\beta$, we deduce:
\[
\Lambda^2 U_\beta \wedge U_\delta = (\Lambda^2 U_\beta \wedge (U_\beta \cap U_\delta)) \oplus (\Lambda^2 U_\beta \wedge v) = \Lambda^3 U_\beta \oplus (\Lambda^2 U_\beta \wedge v).
\]
It is immediate that $\Lambda^3 U_\beta$ is exactly 1-dimensional. Furthermore, since $v \notin U_\beta$, wedging with $v$ injectively maps $\Lambda^2 U_\beta$ into $\Lambda^3 V$, meaning $\Lambda^2 U_\beta \wedge v$ is strictly 3-dimensional. To verify the trivial intersection parity, suppose an element $0 \neq \eta \in \Lambda^3 U_\beta$ satisfies $\eta = \omega \wedge v$ for some $\omega \in \Lambda^2 U_\beta$. Given any $x \in U_\beta$, evaluating $\eta \wedge x = 0$ strictly forces $\omega \wedge x \wedge v = 0$. Since $V = U_\beta \oplus \operatorname{span}(v)$, we must assert $\omega \wedge x = 0$ in $\Lambda^3 U_\beta$ uniformly over all $x \in U_\beta$. The non-degenerate pairing dictates this is only possible if $\omega = 0$, yielding $\eta = 0$, forming a contradiction.

Therefore, the algebraic sum directly establishes itself over $1 + 3 = 4$ dimensions. Because $\dim \Lambda^3 V = 4$, the established span encompasses exactly $\Lambda^3 V$. Enforcing $y \wedge \Omega = 0$ for all valid $\Omega \in \Lambda^3 V$ unconditionally forces $y = 0$.

(ii) If $B = (\beta, \beta, \beta)$, the span corresponding to $w_1 \wedge w_2 \wedge w_3$ converges exclusively to $\Lambda^3 U_\beta$, representing the 1-dimensional volume form bounding $U_\beta$. Resolving $y \wedge \Lambda^3 U_\beta = 0$ structurally enforces $y \in U_\beta$.
\end{proof}

Given $n \ge 5$, for any isolated generic index $\alpha \in [n]$, we explicitly choose a non-monochromatic block $B = (\sigma, \sigma, \tau)$ mapping elements strictly disjoint from $\alpha$ (requiring exactly $3 \le n$ distinct indices). Because the evaluated tuple $(\alpha, \sigma, \sigma, \tau)$ is strictly non-monochromatic, the premise guarantees $\lambda_{\alpha B} \neq 0$. Bounded against Lemma \ref{lem:kernel}(i), evaluating $T_\alpha(y)_B = 0 \implies y = 0$, validating that $T_\alpha$ is universally injective. Its equivalently mapped image $W_\alpha = T_\alpha(U_\alpha) \subset S$ firmly maintains dimension 3. Anchored dynamically against $\dim S \le 4$, Grassmann's formula for the dimension of subspace intersections necessitates:
\[
\dim(W_\alpha \cap W_\mu) = \dim W_\alpha + \dim W_\mu - \dim(W_\alpha + W_\mu) \ge 3 + 3 - 4 = 2 \quad \text{for any } \alpha \neq \mu.
\]

\subsection*{Universal Local Factoring}

Let $E_{\alpha, \mu} = T_\alpha^{-1}(W_\alpha \cap W_\mu) \subset U_\alpha$. Grounded strictly by injectivity, $\dim E_{\alpha, \mu} \ge 2$. For any vector $x \in E_{\alpha, \mu}$, there universally exists a unique vector $y \in U_\mu$ firmly satisfying $T_\alpha(x) = T_\mu(y)$. 
Pivoting on $n \ge 5$, we securely configure a non-monochromatic block $B_0 = (\sigma, \sigma, \tau)$ mutually disjoint from both bounds $\alpha$ and $\mu$ (leveraging exactly $2+2=4 \le n$ indices). Extracting locally outputs $T_\alpha(x)_{B_0} = T_\mu(y)_{B_0}$, mapping identically onto $\lambda_{\alpha B_0} \Psi_{B_0}(x) = \lambda_{\mu B_0} \Psi_{B_0}(y)$. 

Applying the multilinearity of $\Psi_{B_0}$ enforces $\Psi_{B_0}(\lambda_{\alpha B_0} x - \lambda_{\mu B_0} y) = 0$. Validating against $\ker \Psi_{B_0} = \{0\}$ and knowing the corresponding scalars unconditionally correspond to non-monochromatic configurations (thus are non-zero), we extract $\lambda_{\alpha B_0} x = \lambda_{\mu B_0} y$. Structuring $c_{\alpha, \mu} = \lambda_{\alpha B_0} / \lambda_{\mu B_0} \neq 0$, we unconditionally isolate $y = c_{\alpha, \mu} x$. Since $y \in U_\mu$ and $c_{\alpha, \mu} \neq 0$, it implies $x \in U_\mu$. Thus, $E_{\alpha, \mu} \subseteq U_\alpha \cap U_\mu$. Bounding the intersection of two generic 3-dimensional spaces in $\mathbb{R}^4$ caps the dimension at exactly 2, ensuring $E_{\alpha, \mu} = U_\alpha \cap U_\mu$. 

Using the explicit relation $T_\alpha(x) = T_\mu(c_{\alpha, \mu} x)$, we logically evaluate the mappings globally over a generalized tracking block $B$:
\[
T_\alpha(x)_B = T_\mu(c_{\alpha, \mu} x)_B \implies (\lambda_{\alpha B} - c_{\alpha, \mu} \lambda_{\mu B}) \Psi_B(x) = 0 \quad \text{for all } x \in U_\alpha \cap U_\mu.
\]
Because $\dim(U_\alpha \cap U_\mu) = 2$, we mathematically isolate the coefficients $\lambda_{\alpha B} = c_{\alpha, \mu} \lambda_{\mu B}$ by filtering against $\ker \Psi_B$:
\begin{itemize}
    \item If $B$ is non-monochromatic, $\ker \Psi_B = \{0\}$. Consequently, for any valid non-zero $x \in U_\alpha \cap U_\mu$, resolving $\Psi_B(x) \neq 0$ securely enforces $\lambda_{\alpha B} = c_{\alpha, \mu} \lambda_{\mu B}$.
    \item If $B = (\beta, \beta, \beta)$ with $\beta \notin \{\alpha, \mu\}$, Lemma \ref{lem:kernel} forces $\ker \Psi_B = U_\beta$. The generic intersection $(U_\alpha \cap U_\mu) \cap U_\beta$ yields exactly dimension $2 + 3 - 4 = 1$. Since $\dim(U_\alpha \cap U_\mu) = 2$, there exists an element $x \in (U_\alpha \cap U_\mu) \setminus U_\beta$, universally validating $\Psi_B(x) \neq 0$. This rigorously forces $\lambda_{\alpha B} = c_{\alpha, \mu} \lambda_{\mu B}$.
\end{itemize}
Therefore, the mapped equivalence holds cleanly for all valid evaluations $B \notin \{(\alpha, \alpha, \alpha), (\mu, \mu, \mu)\}$.

To decipher the transitive cocycle condition $c_{\alpha, \nu} = c_{\alpha, \mu} c_{\mu, \nu}$ for three distinct variable indices $\alpha, \mu, \nu \in [n]$, we purposefully select a 2-element non-monochromatic block $B_2 = (\rho, \rho, \kappa)$ mutually disjoint from $\alpha, \mu$, and $\nu$. This geometric verification guarantees applicability because $3+2=5 \le n$. Resolving outside monochromatic boundaries yields $\lambda_{\alpha B_2} = c_{\alpha, \mu} \lambda_{\mu B_2}$, $\lambda_{\mu B_2} = c_{\mu, \nu} \lambda_{\nu B_2}$, and $\lambda_{\alpha B_2} = c_{\alpha, \nu} \lambda_{\nu B_2}$. Directly dividing these inherently non-zero quantities verifies the cocycle property $c_{\alpha, \nu} = c_{\alpha, \mu} c_{\mu, \nu}$.

We define $u_1 = 1$ and $u_\alpha = c_{\alpha, 1}$ for $\alpha \ge 2$, meaning $c_{\alpha, \mu} = u_\alpha / u_\mu$. We securely decouple $Y_B = \lambda_{1 B}$ evaluating $B \neq (1,1,1)$, alongside bounds $Y_{1 1 1} = \lambda_{2, 1, 1, 1} / u_2$. This strictly limits coordinates globally as $\lambda_{\alpha B} = u_\alpha Y_B$ over all non-monochromatic tuples $(\alpha, B)$:
\begin{itemize}
    \item Bounding $B \notin \{(1, 1, 1), (\alpha, \alpha, \alpha)\}$, we obtain $\lambda_{\alpha B} = c_{\alpha, 1} \lambda_{1 B} = u_\alpha Y_B$.
    \item Bounding over $B = (1, 1, 1)$, the evaluated tuple $(\alpha, 1, 1, 1)$ strictly mandates non-monochromatic parity, inherently forcing $\alpha \neq 1$. Fixing $\mu = 2$ (valid using $n \ge 5$), resolving $\alpha \neq 2$ outputs $\lambda_{\alpha, 1, 1, 1} = c_{\alpha, 2} \lambda_{2, 1, 1, 1} = (u_\alpha / u_2) \lambda_{2, 1, 1, 1} = u_\alpha Y_{1 1 1}$. For $\alpha = 2$, identity holds trivially.
\end{itemize}

Mirroring sequential deductions identically over equivalent matrix flattenings $M^{(2)}$, $M^{(3)}$, and $M^{(4)}$ guarantees the existence of complementary vectors $v, w, x \in (\mathbb{R}^*)^n$ mapped over spatial tracking tensors $Z, P, Q$, uniformly restricting parameters universally across valid subsets:
\[
\lambda_{\alpha \beta \gamma \delta} = u_\alpha Y_{\beta \gamma \delta} = v_\beta Z_{\alpha \gamma \delta} = w_\gamma P_{\alpha \beta \delta} = x_\delta Q_{\alpha \beta \gamma}.
\]

\subsection*{Global Connectedness of the Valid Configuration Graph}

Let $\mathcal{T} \subset [n]^4$ denote the discrete subset of exclusively non-monochromatic valid parameter multi-tuples. We formulate the universally normalized relational map $H : \mathcal{T} \to \mathbb{R}$ explicitly by:
\[
H(T) = \frac{\lambda_{\alpha \beta \gamma \delta}}{u_\alpha v_\beta w_\gamma x_\delta},
\]
evaluated strictly over $T = (\alpha, \beta, \gamma, \delta) \in \mathcal{T}$. Leveraging our preceding factorizations cleanly parses $H(T) = \frac{Y_{\beta \gamma \delta}}{v_\beta w_\gamma x_\delta}$, which is manifestly independent of the leading coordinate $\alpha$. Consequently, $H(T)$ functionally persists invariantly under dynamic shifting of the first localized coordinate element natively assuming the newly formed tuple remains bounded strictly within $\mathcal{T}$. Extrapolating symmetric multilinear independence logically dictates that $H(T)$ is invariant across alterations to any single isolated coordinate, provided the intermediate tuples strictly evaluate inside $\mathcal{T}$.

We conceptualize $\mathcal{T}$ topologically as a configuration graph network connecting multi-tuples differing exactly by a single localized coordinate. The map $H(T)$ evaluates trivially to a constant value across any connected component of this graph. We now strictly establish that $\mathcal{T}$ is entirely globally connected. Let $T \in \mathcal{T}$. Because $T$ is non-monochromatic, it contains at most 3 identical coordinate values.
\begin{enumerate}
    \item If $T$ contains exactly 3 identical coordinates (e.g., matching $(a, a, a, b)$ with $a \neq b$), we can shift one of the identical coordinates to a uniquely evaluated constant $c \notin \{a, b\}$. Since $n \ge 5 \ge 3$, such a generic $c$ is universally valid. The resulting adjacent tuple (e.g., $(c, a, a, b)$) inherently remains within $\mathcal{T}$ and correctly houses at most 2 identical coordinates. Thus, every bounded tuple with 3 identical constraints trivially bounds adjacent to a tuple harboring at most 2 identical coordinates.
    \item Over localized constraints containing at most 2 identical elements, we sequentially alter coordinates one-by-one, scanning left-to-right, to strictly match a non-intersecting distinctly evaluated target sequence, specifically $(1, 2, 3, 4)$. Because $n \ge 5$, such a discrete target sequence is naturally guaranteed to exist. At any traversal step $k \in \{1, 2, 3, 4\}$, the intermediate tuple identically comprises a successfully mapped prefix of length $k$ containing strictly unique target elements, and an unchanged trailing suffix of length $4-k$ natively drawn from the original values of $T$. 
    Because the prefix sequence is constrained to strictly distinct entries, any given element can appear at most \emph{once} across the prefix. Since the foundational original tuple $T$ maintained at most 2 identical evaluations, any mapped element can occur at most \emph{twice} within the unchanged tracking suffix. Consequently, the maximum element occurrence frequency across the entire intermediate tuple rigorously bounds to $1 \text{ (from the distinct prefix)} + 2 \text{ (from the unchanged suffix)} = 3$. This combinatorially ensures that no intermediate topological step can theoretically collapse to a frequency of 4 (identically monochromatic parity), guaranteeing that every sequence transition evaluates safely and dynamically within the strict confines of the valid graph $\mathcal{T}$.
\end{enumerate}

Since every multi-tuple topologically connects directly to the valid boundary path $(1, 2, 3, 4)$, the entire configuration graph $\mathcal{T}$ is globally connected. Therefore, the function $H(T)$ equates universally to a non-zero symmetric mapping scalar $K$ exclusively over all strictly valid configurations $T \in \mathcal{T}$. Scaling out the identity via $u_\alpha \leftarrow K u_\alpha$ rigorously establishes that $\lambda_{\alpha \beta \gamma \delta} = u_\alpha v_\beta w_\gamma x_\delta$ identically bridges all generic, non-monochromatic parameterizations. This fully satisfies Property 3 and conclusively completes the proof.

\end{solution}
\subsection{Problem 10}
\label{subsect:p10}

\begin{problem}{Fullproof \# 10}
    Given a $d$-way tensor $\mathcal{T} \in \mathbb{R}^{n_1 \times n_2 \times \cdots \times n_d}$ 
such that the data is unaligned (meaning the tensor $\mathcal{T}$ has missing entries),
we consider the problem of computing a CP decomposition of rank $r$ where some modes are infinite-dimensional and constrained to be in a Reproducing Kernel Hilbert Space (RKHS). 
We want to solve this using an alternating optimization approach, and our question is focused on the mode-$k$ subproblem for an infinite-dimensional mode. 
For the subproblem, then CP factor matrices 
$A_1, \dots, A_{k-1}, A_{k+1}, \dots, A_d$ are fixed, and we are solving for $A_k$.

Our notation is as follows.
Let $N = \prod_i n_i$ denote the product of all sizes.
Let $n \equiv n_k$ be the size of mode $k$, let
$M = \prod_{i\neq k} n_i$ be the product of all dimensions except $k$, and assume $n \ll M$.
Since the data are unaligned, this means only a subset of $\mathcal{T}$'s entries are observed, and we let $q \ll N$ denote the number of observed entries.
We let $T \in \mathbb{R}^{n \times M}$ denote the mode-$k$ unfolding of the tensor $\mathcal{T}$ with all missing entries set to zero.
The $\vecop$ operations creates a vector from a matrix by stacking its columns,
and we let $S \in \mathbb{R}^{N \times q}$ denote the selection matrix (a subset of the $N \times N$ identity matrix) such that $S^T \vecop(T)$ selects the $q$ known entries of the tensor $\mathcal{T}$ from the vectorization of its mode-$k$ unfolding.
We let $Z = A_d \odot \cdots \odot A_{k+1} \odot A_{k-1} \odot \cdots \odot A_1 \in \mathbb{R}^{M \times r}$ be the Khatri-Rao product of the factor matrices corresponding to all modes except mode $k$.
We let $B = TZ$ denote the MTTKRP of the tensor $\mathcal{T}$ and Khatri-Rao product $Z$.

We assume $A_k = KW$ where
$K \in \mathbb{R}^{n \times n}$ denotes the psd RKHS kernel matrix for mode $k$.
The matrix $W$ of size $n \times r$ is the unknown for which we must solve. 
The system to be solved is
\begin{equation} 
	\left[
    (Z \otimes K)^T S
    S^T (Z \otimes K)
    + \lambda (I_r \otimes K) 
  \right] \vecop(W)
	= (I_r \otimes K) 
	\vecop( B ). \nonumber 
\end{equation}
Here, $I_r$ denotes the $r \times r$ identity matrix.
This is a system of size $nr \times nr$
Using a standard linear solver costs $O(n^3 r^3)$, 
and explicitly forming the matrix is an additional expense.

Explain how an iterative preconditioned conjugate gradient linear solver can be used to solve this problem more efficiently. Explain the method and choice of preconditioner. Explain in detail how the matrix-vector products are computed and why this works. Provide complexity analysis. 
We assume $n,r < q \ll N$. Avoid any computation of order $N$.

\end{problem}

\begin{solution}{Model Response (Agent B)}
    
\section*{Problem Formulation}

Given a $d$-way tensor $\mathcal{T} \in \mathbb{R}^{n_1 \times n_2 \times \cdots \times n_d}$ such that the data is incomplete (meaning the tensor $\mathcal{T}$ has missing entries), we consider the problem of computing a Canonical Polyadic (CP) decomposition of rank $r$. We assume that the true underlying generative functions for some modes are continuous and reside in an infinite-dimensional Reproducing Kernel Hilbert Space (RKHS), whereas the observed tensor $\mathcal{T}$ represents a finite-dimensional discretization (or evaluation) of these functions at discrete points. 

We solve this decomposition problem using an alternating optimization approach. Our analysis focuses on the mode-$k$ subproblem for such an RKHS-constrained mode. For the subproblem, the CP factor matrices $A_1, \dots, A_{k-1}, A_{k+1}, \dots, A_d$ are fixed, and we are solving for the factor matrix $A_k$.

\section*{Notation and Model Setup}

Let $N = \prod_{i=1}^d n_i$ denote the product of all tensor dimensions. Let $n \equiv n_k$ be the size of mode $k$, and let $M = \prod_{i\neq k} n_i$ be the product of all dimensions except $k$, where we assume $n \ll M$. Since the data is incomplete, only a subset of $\mathcal{T}$'s entries are observed. We let $q \ll N$ denote the number of observed entries. Let $T \in \mathbb{R}^{n \times M}$ denote the mode-$k$ unfolding of the tensor $\mathcal{T}$ with all missing entries initialized to zero. 

The $\operatorname{vec}(\cdot)$ operation creates a vector from a matrix by stacking its columns. We let $S \in \mathbb{R}^{N \times q}$ denote the selection matrix (composed of a subset of the columns of the $N \times N$ identity matrix) such that $S^T \operatorname{vec}(T)$ selects the $q$ known entries of the tensor $\mathcal{T}$ from the vectorization of its mode-$k$ unfolding. 

Let $Z = A_d \odot \cdots \odot A_{k+1} \odot A_{k-1} \odot \cdots \odot A_1 \in \mathbb{R}^{M \times r}$ be the Khatri-Rao product of the factor matrices corresponding to all modes except mode $k$. To strictly avoid $\mathcal{O}(N)$ memory and computational limits, $Z$ is never explicitly constructed in full. We let $B = TZ \in \mathbb{R}^{n \times r}$ denote the Matricized Tensor Times Khatri-Rao Product (MTTKRP) of the tensor $\mathcal{T}$ and the conceptually defined Khatri-Rao product $Z$.

We parameterize the factor matrix as $A_k = KW$, where $K \in \mathbb{R}^{n \times n}$ denotes the symmetric, positive semi-definite RKHS kernel matrix for mode $k$ evaluated at the discretization points. The matrix $W \in \mathbb{R}^{n \times r}$ is the unknown variable for which we must solve.

\section*{Derivation of the Linear System}

To find $W$, we minimize a regularized least-squares objective function over the observed entries. Setting the gradient of this objective with respect to $\operatorname{vec}(W)$ to zero yields a linear system. The right-hand side (RHS) of this system stems from the gradient of the data fidelity term, which initially takes the form $(Z \otimes K)^T S S^T \operatorname{vec}(T)$.

We rigorously simplify this expression. Since the missing entries in $T$ are initialized to zero, applying the projection matrix $S S^T$ acts identically on the non-zero entries, yielding $S S^T \operatorname{vec}(T) = \operatorname{vec}(T)$. Consequently, the RHS evaluates as:
\begin{align*}
    (Z \otimes K)^T S S^T \operatorname{vec}(T) &= (Z \otimes K)^T \operatorname{vec}(T) \\
    &= (Z^T \otimes K^T) \operatorname{vec}(T).
\end{align*}
Because the kernel matrix $K$ is symmetric ($K^T = K$), standard properties of the Kronecker product allow us to rewrite this cleanly as:
\begin{equation*}
    (Z^T \otimes K) \operatorname{vec}(T) = \operatorname{vec}(K T Z).
\end{equation*}
Substituting the MTTKRP definition $B = TZ$, we arrive at $\operatorname{vec}(KB) = (I_r \otimes K) \operatorname{vec}(B)$. 

Thus, the linear system to be solved is:
\begin{equation} \label{eq:linear_system}
    \left[ (Z \otimes K)^T S S^T (Z \otimes K) + \lambda (I_r \otimes K) \right] \operatorname{vec}(W) = (I_r \otimes K) \operatorname{vec}( B ),
\end{equation}
where $I_r$ denotes the $r \times r$ identity matrix and $\lambda > 0$ is the regularization parameter. 

This is a symmetric positive semi-definite system of size $nr \times nr$. Using a standard direct linear solver requires $\mathcal{O}(n^3 r^3)$ operations. Furthermore, explicitly forming the system matrix entails an additional and prohibitive computational expense. Given the operational assumption that $n, r < q \ll N$, it is imperative to solve this system without performing any computations of order $N$.

\section*{Preconditioned Conjugate Gradient Solver}

To efficiently solve the large-scale system in \eqref{eq:linear_system}, we employ the Preconditioned Conjugate Gradient (PCG) method. As an iterative algorithm, PCG requires only the action of the system matrix on a vector, bypassing the need to explicitly construct the $nr \times nr$ coefficient matrix.

\subsection*{Efficient Matrix-Vector Products}

The core computational step in PCG is evaluating the matrix-vector product $y = \mathcal{H} \operatorname{vec}(V)$ at each iteration, where $\mathcal{H} = (Z \otimes K)^T S S^T (Z \otimes K) + \lambda (I_r \otimes K)$ is the system matrix, and $V \in \mathbb{R}^{n \times r}$ is a reshaped intermediate dense matrix representing the search direction. 

We evaluate the product $y$ in an $\mathcal{O}(N)$-free manner through the following sequence of operations:

\begin{enumerate}
    \item \textbf{First Kernel Multiplication:} Compute $U = KV$. Since $K \in \mathbb{R}^{n \times n}$ and $V \in \mathbb{R}^{n \times r}$, this dense matrix multiplication requires $\mathcal{O}(n^2 r)$ operations.
    
    \item \textbf{Sparse Residual Evaluation:} We observe that $(Z \otimes K) \operatorname{vec}(V) = \operatorname{vec}(KVZ^T) = \operatorname{vec}(UZ^T)$. The operation $S S^T \operatorname{vec}(UZ^T)$ extracts the entries of the dense $n \times M$ matrix $UZ^T$ strictly at the $q$ observed indices specified by $S$, effectively mapping all unobserved entries to zero. 
    
    Instead of computing the full dense $n \times M$ matrix $UZ^T$, we evaluate the inner products of the corresponding rows of $U$ and $Z$ exclusively for the $q$ observed entries. Crucially, to strictly avoid an $\mathcal{O}(N)$ memory and time bottleneck, the full Khatri-Rao product matrix $Z \in \mathbb{R}^{M \times r}$ is never explicitly formed. The required rows of $Z$ (corresponding to the multi-indices of the $q$ observed entries) are evaluated \textit{on-the-fly} from the underlying CP factor matrices $A_i$ ($i \neq k$). Assuming the tensor order $d$ is a small constant, forming a single row of $Z$ takes $\mathcal{O}(r)$ operations. This dynamically generates a sparse $n \times M$ matrix $Y$ containing exactly $q$ non-zero elements. Generating the necessary rows of $Z$ dynamically and evaluating these $q$ specific entries requires strictly $\mathcal{O}(qr)$ operations.
    
    \item \textbf{Sparse Matrix Multiplication:} Next, leveraging the symmetry of $K$ ($K^T = K$), we apply the transpose operator $(Z \otimes K)^T = Z^T \otimes K^T = Z^T \otimes K$ to the sparse vector $\operatorname{vec}(Y)$. By the standard Kronecker product vectorization identity $\operatorname{vec}(ABC) = (C^T \otimes A)\operatorname{vec}(B)$, applying this operator to $\operatorname{vec}(Y)$ translates to the mathematical operation $\operatorname{vec}(KYZ)$. 
    
    To compute this efficiently, we first evaluate the intermediate unvectorized product $P = YZ \in \mathbb{R}^{n \times r}$. Because $Y$ possesses only $q$ non-zeros, its right-multiplication only requires the rows of $Z$ corresponding to its non-zero columns. Rather than instantiating $Z$ as a full matrix, we employ standard sparse MTTKRP techniques: we dynamically evaluate only these necessary rows of $Z$ on-the-fly (or reuse them if cached from the previous step). This sparse matrix multiplication cleanly circumvents the $\mathcal{O}(N)$ instantiation of $Z$ and requires exactly $\mathcal{O}(qr)$ operations.
    
    \item \textbf{Second Kernel Multiplication:} We left-multiply $P$ by $K$ to compute $KP \in \mathbb{R}^{n \times r}$. This dense matrix multiplication requires $\mathcal{O}(n^2 r)$ operations.
    
    \item \textbf{Regularization Addition:} The Tikhonov regularization term simplifies to $\lambda (I_r \otimes K) \operatorname{vec}(V) = \operatorname{vec}(\lambda K V) = \operatorname{vec}(\lambda U)$. 
\end{enumerate}

The final evaluated vector $y$ is cleanly obtained via the addition:
\begin{equation*}
    y = \operatorname{vec}(KP + \lambda U).
\end{equation*}

\subsection*{Jacobi Preconditioner}

The convergence rate of PCG strongly depends on the condition number of the system matrix. RKHS kernel matrices frequently exhibit rapidly decaying eigenvalues, and the non-uniform sampling induced by $S S^T$ further degrades system conditioning. To accelerate convergence, we utilize a Jacobi (diagonal) preconditioner $\mathcal{M}$, defined as the inverse of the diagonal elements of $\mathcal{H}$.

The diagonal elements of $\mathcal{H}$ can be extracted efficiently without forming the full matrix. Let $\tilde{K} = K \circ K \in \mathbb{R}^{n \times n}$, where $\circ$ denotes the Hadamard (element-wise) product. Conceptually, let $\tilde{Z} = Z \circ Z \in \mathbb{R}^{M \times r}$ denote the element-wise square of the Khatri-Rao product. Let $\Omega \in \mathbb{R}^{n \times M}$ be the sparse binary indicator matrix of the observed entries, such that $\Omega_{i, j} = 1$ if the entry $(i, j)$ is observed, and $0$ otherwise. 

We now rigorously derive the diagonal of the data-fidelity term $(Z \otimes K)^T S S^T (Z \otimes K)$. For a generic matrix $Q = Z \otimes K$ and a diagonal weight matrix $\mathcal{W} = S S^T$, the diagonal of the quadratic form $Q^T \mathcal{W} Q$ (expressed as a column vector) is given by the identity $(Q \circ Q)^T \operatorname{diag}(\mathcal{W})$. Since the diagonal of the projection matrix $S S^T$ evaluates exactly to the vectorized observation mask $\operatorname{vec}(\Omega)$, we can expand this expression using the mixed-product property of the Kronecker and Hadamard products. Letting $\operatorname{vec}(D_{\text{data}})$ denote the diagonal of the data-fidelity term reshaped as a vector, we have:
\begin{align*} 
\operatorname{vec}(D_{\text{data}}) &= \left((Z \otimes K) \circ (Z \otimes K)\right)^T \operatorname{vec}(\Omega) \\
&= \left((Z \circ Z) \otimes (K \circ K)\right)^T \operatorname{vec}(\Omega) \\
&= (\tilde{Z}^T \otimes \tilde{K}^T) \operatorname{vec}(\Omega).
\end{align*}
Because the kernel matrix $K$ is symmetric, its Hadamard square $\tilde{K}$ is also symmetric ($\tilde{K}^T = \tilde{K}$). Thus, the expression simplifies to $(\tilde{Z}^T \otimes \tilde{K}) \operatorname{vec}(\Omega)$. 

Finally, utilizing the standard Kronecker product vectorization identity $(B^T \otimes A) \operatorname{vec}(X) = \operatorname{vec}(A X B)$, we substitute $B = \tilde{Z}$, $A = \tilde{K}$, and $X = \Omega$ to directly evaluate the expression as $\operatorname{vec}(\tilde{K} \Omega \tilde{Z})$. Reshaping this vector back into an $n \times r$ matrix rigorously yields the stated analytical form:
\begin{equation*}
    D_{\text{data}} = \tilde{K} (\Omega \tilde{Z}).
\end{equation*}

To preserve the $\mathcal{O}(N)$-free complexity bound, the matrix $\tilde{Z}$ is never fully formed. Instead, the product $\Omega \tilde{Z}$ is evaluated by computing only the required rows of $\tilde{Z}$ corresponding to the non-zero columns of $\Omega$. These rows are obtained on-the-fly by squaring the elements of the dynamically computed rows of $Z$. Computing the sparse-dense matrix product $\Omega \tilde{Z}$ in this dynamic fashion costs $\mathcal{O}(qr)$, and the subsequent dense multiplication by $\tilde{K}$ costs $\mathcal{O}(n^2 r)$.

The diagonal of the regularization term $\lambda (I_r \otimes K)$, when reshaped into an $n \times r$ matrix, is simply $D_{\text{reg}} = \lambda \operatorname{diag}(K) \mathbf{1}_r^T$, where $\operatorname{diag}(K)$ is the column vector containing the diagonal entries of $K$, and $\mathbf{1}_r \in \mathbb{R}^r$ is the vector of all ones.

Thus, the Jacobi preconditioner is formed by taking the element-wise inverse of $D = D_{\text{data}} + D_{\text{reg}}$. Constructing this preconditioner upfront requires a one-time cost of $\mathcal{O}(qr + n^2 r)$.

\subsection*{Complexity Analysis}

By meticulously structuring the computations as outlined above and leveraging the on-the-fly row evaluation of Khatri-Rao products, the evaluation of the matrix-vector products and the preconditioner rigorously circumvents the prohibitive $\mathcal{O}(N)$ operations that would result from forming the full Khatri-Rao product or explicitly processing the full unfolded tensor. The overall computational complexity per PCG iteration is strictly bounded by the sum of the individual algorithmic steps:
\begin{equation*}
    \mathcal{O}(n^2 r) + \mathcal{O}(qr) + \mathcal{O}(qr) + \mathcal{O}(n^2 r) + \mathcal{O}(nr) = \mathcal{O}(qr + n^2 r).
\end{equation*}
Given the operational regime where $n, r < q \ll N$, this strategy guarantees a massive reduction in computational overhead, establishing a highly scalable and optimal method for updating infinite-dimensional RKHS modes within large-scale, incomplete CP tensor decompositions.
\end{solution}

\subsection{Independent Public Deep Think Evaluation for Problem 10}
\label{sec:public_dt_p10}

For Problem 10, Aryan Mokhtari and David Woodruff conducted an independent, human-guided evaluation using the publicly available Gemini 3 Deep Think model. The resulting solution was also sent to the FirstProof authors in the same February 13 e-mail prior to the deadline. We felt the solution was interesting because this model is publicly available and performs significantly less inference scaling than \textit{Aletheia}. The authors generated several samples and asked the model to concatenate and compare the two strongest outputs found, yielding an optimized solution for Problem 10. 

While \textit{Aletheia}'s autonomous Agent B solution maintains an $\mathcal{O}(qr)$ term per Preconditioned Conjugate Gradient (PCG) iteration, the concatenated public model output utilizes a pre-aggregated geometry approach (via per-row Gram accumulators) with an exact block-Jacobi preconditioner. This achieves a per-iteration cost of $\mathcal{O}(n^2r + nr^2)$, strictly removing the dependency on $q$ during the iterative loop. For dense observation regimes where $q = \omega(n)$, this provides a stronger theoretical complexity bound than \agentb{}, as well as the official human author solution ($\mathcal{O}(qnr)$).
We note that \agenta{} also autonomously discovered this exact optimal $O(n^2r+nr^2)$ bound and block-Jacobi preconditioner (its solution is available online \href{https://icarm.zulipchat.com/#narrow/channel/568090-first-proof/topic/Problem.2010.20--.20Aletheia/with/574570445}{here}).
%nd other leading AI attempts such as OpenAI's solution ($\mathcal{O}(qr + n^2r + nr^2)$).

Below we provide the output, as generated, compared, and concatenated  by the public Gemini 3 Deep Think model. We format the prompt and response to match the style of the other appendices.

\vspace{1em}
\begin{tcolorbox}[colback=blue!5!white, colframe=blue!80!black, arc=0pt, outer arc=0pt, boxrule=1pt, breakable, title=\textbf{FirstProof \# 10}, fonttitle=\bfseries]
Given a $d$-way tensor $\mathcal{T} \in \mathbb{R}^{n_1 \times n_2 \times \dots \times n_d}$ such that the data is unaligned (meaning the tensor $\mathcal{T}$ has missing entries), we consider the problem of computing a CP decomposition of rank $r$ where some modes are infinite-dimensional and constrained to be in a Reproducing Kernel Hilbert Space (RKHS). We want to solve this using an alternating optimization approach, and our question is focused on the mode-$k$ subproblem for an infinite-dimensional mode. For the subproblem, then CP factor matrices $A_1, \dots, A_{k-1}, A_{k+1}, \dots, A_d$ are fixed, and we are solving for $A_k$.

Our notation is as follows. Let $N = \prod_i n_i$ denote the product of all sizes. Let $n \equiv n_k$ be the size of mode $k$, let $M = \prod_{i \neq k} n_i$ be the product of all dimensions except $k$, and assume $n \ll M$. Since the data are unaligned, this means only a subset of $\mathcal{T}$'s entries are observed, and we let $q \ll N$ denote the number of observed entries. We let $T \in \mathbb{R}^{n \times M}$ denote the mode-$k$ unfolding of the tensor $\mathcal{T}$ with all missing entries set to zero. The vec operations creates a vector from a matrix by stacking its columns, and we let $S \in \mathbb{R}^{N \times q}$ denote the selection matrix (a subset of the $N \times N$ identity matrix) such that $S^T \operatorname{vec}(T)$ selects the $q$ known entries of the tensor $\mathcal{T}$ from the vectorization of its mode-$k$ unfolding. We let $Z = A_d \odot \cdots \odot A_{k+1} \odot A_{k-1} \odot \cdots \odot A_1 \in \mathbb{R}^{M \times r}$ be the Khatri-Rao product of the factor matrices corresponding to all modes except mode $k$. We let $B = TZ$ denote the MTTKRP of the tensor $\mathcal{T}$ and Khatri-Rao product $Z$.

We assume $A_k = KW$ where $K \in \mathbb{R}^{n \times n}$ denotes the psd RKHS kernel matrix for mode $k$. The matrix $W$ of size $n \times r$ is the unknown for which we must solve. The system to be solved is 

\[ [(Z \otimes K)^T S S^T (Z \otimes K) + \lambda(I_r \otimes K)] \operatorname{vec}(W) = (I_r \otimes K) \operatorname{vec}(B). \]

Here, $I_r$ denotes the $r \times r$ identity matrix. This is a system of size $nr \times nr$. Using a standard linear solver costs $\mathcal{O}(n^3r^3)$, and explicitly forming the matrix is an additional expense. 

Explain how an iterative preconditioned conjugate gradient linear solver can be used to solve this problem more efficiently. Explain the method and choice of preconditioner. Explain in detail how the matrix-vector products are computed and why this works. Provide complexity analysis. We assume $n, r < q \ll N$. Avoid any computation of order $N$.
\end{tcolorbox}

\vspace{1em}
\begin{tcolorbox}[colback=white, colframe=blue!80!black, arc=0pt, outer arc=0pt, boxrule=1pt, breakable, title=\textbf{Model Response (Public Deep Think)}, fonttitle=\bfseries]
\begin{center}
{\Large \textbf{Efficient Preconditioned Conjugate Gradient Solvers \\ for Unaligned RKHS-Constrained CP Subproblems}} \\ \vspace{0.5em}
{\large A Comparative Algorithmic Analysis}
\end{center}

\vspace{1em}
\noindent\textbf{Abstract.}
This document presents and compares two highly optimized, matrix-free Preconditioned Conjugate Gradient (PCG) solvers for the exact same problem: solving the mode-$k$ subproblem in an unaligned, RKHS-constrained CP tensor decomposition. Both methods successfully bypass the prohibitive ambient dimensions $\mathcal{O}(N)$ and $\mathcal{O}(M)$ by exploiting Kronecker algebra. However, they diverge sharply in how they handle unaligned data geometry, their preconditioner design, and their variable tracking inside the Krylov loop. A comprehensive comparative analysis highlights the specific advantages of each solver.

\vspace{1em}
\noindent\textbf{\large 1. Introduction and Problem Formulation}
\vspace{0.5em}

\noindent We consider the alternating optimization subproblem for an infinite-dimensional mode $k$ constrained to a Reproducing Kernel Hilbert Space (RKHS). The factor matrix is $A_k = KW$, where
$K \in \mathbb{R}^{n\times n}$ is the symmetric positive semi-definite (psd) RKHS kernel matrix and $W \in \mathbb{R}^{n\times r}$ is
the unknown weight matrix. The symmetric linear system to solve is:

\begin{equation}
\underbrace{\left[ (Z \otimes K)^T S S^T (Z \otimes K) + \lambda(I_r \otimes K) \right]}_{\mathcal{H}} \operatorname{vec}(W) = \underbrace{(I_r \otimes K) \operatorname{vec}(B)}_{c}, \label{eq:1}
\end{equation}

where $\mathcal{H} \in \mathbb{R}^{nr\times nr}$. Let $\Omega = \{1, \dots, q\}$ index the observed entries. For each observation $m \in \Omega$,
we are given its target tensor value $y_m \in \mathbb{R}$, its coordinate in mode $k$ as $i_m \in \{1, \dots, n\}$, and
its flattened spatial coordinate in the other modes as $j_m \in \{1, \dots, M\}$.

Directly assembling $\mathcal{H}$ requires $\mathcal{O}(qn^2r^2)$ memory, and direct dense solvers demand $\mathcal{O}(n^3r^3)$
time. Furthermore, explicitly evaluating vectors of size $N$ or the Khatri-Rao product $Z \in \mathbb{R}^{M\times r}$
introduces prohibitive $\mathcal{O}(N)$ and $\mathcal{O}(M)$ complexities. Instead, we deploy iterative Preconditioned Conjugate Gradient (PCG) linear solvers acting directly on the matrix representation $W \in \mathbb{R}^{n\times r}$.

\vspace{1em}
\noindent\textbf{\large 2. Comparative Analysis of the Two Approaches}
\vspace{0.5em}

\noindent Both Approach 1 and Approach 2 (detailed in Sections 3 and 4) present highly optimized,
matrix-free PCG solvers. However, they represent two fundamentally different algorithmic
philosophies. Here is a detailed comparative analysis:

\vspace{0.5em}
\noindent\textbf{2.1. Data Geometry \& Matrix-Vector Product (MVP)}
\vspace{0.2em}

\noindent The most significant computational difference lies in how they execute the exact Matrix-Vector
Product against the $q$ observed tensor entries.

\begin{itemize}
    \item \textbf{Approach 1 (Pre-Aggregated Geometry):}
    \begin{itemize}
        \item \textbf{Approach:} Aggregates the unaligned data upfront. It computes $n$ independent $r\times r$
symmetric matrices $E^{(i)}$ during the setup phase.
        \item \textbf{Iteration Cost:} $\mathcal{O}(n^2r + nr^2)$.
        \item \textbf{Implication:} By pre-computing these small matrices, the number of observations $q$
is completely removed from the PCG iteration loop. The MVP becomes strictly
independent of the dataset size.
    \end{itemize}
    \item \textbf{Approach 2 (On-the-fly Sparse Accumulation):}
    \begin{itemize}
        \item \textbf{Approach:} Bypasses pre-aggregation. It keeps the $q$ spatial observation vectors
($Z_\Omega$) in memory and loops over all $q$ points to compute dot products and accumulate
gradients at every single iteration.
        \item \textbf{Iteration Cost:} $\mathcal{O}(n^2r + qr)$.
        \item \textbf{Implication:} This limits per-iteration scalability. In typical tensor problems, the
number of observations is massive ($q \gg n \gg r$). Forcing the iteration loop to process
$q$ items every step creates a massive computational bottleneck.
    \end{itemize}
\end{itemize}

\vspace{0.5em}
\noindent\textbf{2.2. Preconditioner Design}
\vspace{0.2em}

\noindent The linear system is severely ill-conditioned. The two approaches target completely different
sources of this ill-conditioning.

\begin{itemize}
    \item \textbf{Approach 1 (Row-wise Block-Jacobi):}
    \begin{itemize}
        \item \textbf{Design:} Explicitly extracts the $n$ distinct $r \times r$ diagonal blocks of the full Hessian
operator $\mathcal{H}$ and factorizes them via Cholesky decomposition ($L_iL_i^T$).
        \item \textbf{Focus:} It specifically addresses the collinearity between the $r$ rank-one components (a common issue in CP decomposition) while also incorporating the RKHS
kernel and the missing data geometry.
        \item \textbf{Cost:} Demands a heavier setup phase: $\mathcal{O}(n^2r^2)$ to construct the blocks and $\mathcal{O}(nr^3)$
to factorize them.
    \end{itemize}
    \item \textbf{Approach 2 (Kernel Preconditioner $P = I_r \otimes K$):}
    \begin{itemize}
        \item \textbf{Design:} Preconditions purely with the RKHS kernel matrix.
        \item \textbf{Focus:} It mathematically addresses the eigenvalue decay typical of many RKHS kernels, which can help bound the condition number. However, it completely ignores the spatial collinearity of the CP components and the
missing data mask.
        \item \textbf{Cost:} Effectively zero setup time.
    \end{itemize}
\end{itemize}

\vspace{0.5em}
\noindent\textbf{2.3. PCG Loop Mechanics \& Krylov Tracking}
\vspace{0.2em}

\noindent How the updates are applied within the Krylov subspace loop shows a trade-off between numerical stability and algebraic optimization.

\begin{itemize}
    \item \textbf{Approach 1 (Standard PCG):}
    \begin{itemize}
        \item Executes a standard, unconditionally stable PCG algorithm.
        \item To evaluate the residual, it requires \textbf{two} dense matrix multiplications by $K$ per
iteration: $U = KP_k$ and $V_k = KY + \lambda U$.
    \end{itemize}
    \item \textbf{Approach 2 (Inverse-Free Tracking):}
    \begin{itemize}
        \item Standard PCG would require computing $K^{-1}$ for its chosen preconditioner—an
$\mathcal{O}(n^3)$ operation that is numerically disastrous for nearly singular kernels. Approach
2 instead uses a brilliant \textbf{``inverse-free'' algebraic trick}. By analytically tracking
pseudo-residuals ($V_k = KD_k$ and $\tilde{Z}_k = KZ_k$), the $K^{-1}$ inversion perfectly cancels
out.
        \item This elegant tracking reduces the bottleneck to just \textbf{one} dense kernel multiplication
per iteration ($M_k = KU_k$).
        \item \textit{Drawback:} Implicitly tracking variables via recurrence relations accumulates floating-point drift over time, forcing the algorithm to periodically ``refresh'' its state.
    \end{itemize}
\end{itemize}

\vspace{0.5em}
\noindent\textbf{2.4. Complexity \& Performance Summary}
\vspace{0.5em}

\begin{center}
\begin{tabular}{lcc}
\hline\hline
\textbf{Metric} & \textbf{Approach 1 (Pre-Aggregated)} & \textbf{Approach 2 (Inverse-Free)} \rule{0pt}{3ex} \rule[-1.5ex]{0pt}{0pt} \\
\hline
\textbf{Setup Cost} & $\mathcal{O}(qdr + qr^2 + \mathbf{n^2r^2 + nr^3})$ & $\mathcal{O}(qdr + \mathbf{n^2r})$ \textit{(Faster Setup)} \rule{0pt}{3ex} \\
\textbf{Iteration Cost} & $\mathcal{O}(n^2r + \mathbf{nr^2})$ \textit{(Independent of $q$)} & $\mathcal{O}(n^2r + \mathbf{qr})$ \textit{(Slower Iterations)} \\
\textbf{Dense $K$ Mults.} & \textbf{2} per iteration & \textbf{1} per iteration \\
\textbf{Preconditioner} & Targets CP rank collinearity & Targets Kernel spectral decay \rule[-1.5ex]{0pt}{0pt} \\
\hline\hline
\end{tabular}
\end{center}
\vspace{0.5em}

\noindent\textbf{2.5. Final Verdict}
\vspace{0.2em}

\noindent\textbf{Approach 1 (Pre-Aggregated)} is practically superior for large-scale, real-world
datasets. By shifting the computational burden to the setup phase (computing $E^{(i)}$), it
achieves an iteration cost that is entirely independent of the dataset size $q$. If your dataset
has millions of observations, Approach 2 will bottleneck hard. Furthermore, its Block-Jacobi preconditioner is a much more robust choice for CP tensor models, meaning it will likely require
far fewer total iterations to converge.

\noindent\textbf{Approach 2 (Inverse-Free)} is mathematically elegant and better suited for lightweight
setups or huge ranks. Its ``inverse-free'' tracking is a beautiful piece of numerical linear algebra that bounds the kernel's condition number without requiring any $\mathcal{O}(n^3)$ matrix
factorizations or explicit inversions. It is highly advantageous if memory is critically constrained,
setup time must be strictly minimized, or the CP rank $r$ is so large that Approach 1's $\mathcal{O}(nr^3)$
Cholesky factorizations become computationally prohibitive.

\vspace{1em}
\noindent\textbf{\large 3. Approach 1: Pre-Aggregated Geometry \& Block-Jacobi Preconditioner}
\vspace{0.5em}

\noindent\textbf{3.1. Pre-computation and Data Setup}
\vspace{0.2em}

\noindent To strictly avoid $\mathcal{O}(M)$ operations, we completely abandon the explicit formulation of $Z$. The $m$-th observation only requires its corresponding row in the Khatri-Rao product, $z^{(m)} = (Z_{j_m,:})^T \in \mathbb{R}^r$. This vector is computed cleanly as the Hadamard product of the respective rows of the $d-1$ fixed factor matrices in $\mathcal{O}(dr)$ time. We precompute these $q$ vectors in $\mathcal{O}(qdr)$ time, bypassing the ambient dimension $M$.

We capture the unaligned data geometry by defining $n$ independent symmetric matrices $E^{(i)} \in \mathbb{R}^{r\times r}$ corresponding to the $n$ slices of mode $k$. Let $\Omega_i = \{m \in \Omega \mid i_m = i\}$ be the set of observations falling in the $i$-th slice. We construct:

\begin{equation}
E^{(i)} = \sum_{m \in \Omega_i} z^{(m)} (z^{(m)})^T, \quad \text{for } i = 1 \dots n. \label{eq:2}
\end{equation}

The RHS matrix $C \in \mathbb{R}^{n\times r}$, such that $\operatorname{vec}(C) = c$, is computed completely sparsely. Since the tensor unfolding $T$ has missing entries set to zero, $B_{i,:} = \sum_{j=1}^M T_{i,j} Z_{j,:} = \sum_{m \in \Omega_i} y_m (z^{(m)})^T$. We compute $B$ in $\mathcal{O}(qr)$ time and evaluate $C = KB$ in $\mathcal{O}(n^2r)$ time.

\vspace{0.5em}
\noindent\textbf{3.2. Exact Matrix-Vector Product (MVP) avoiding $\mathcal{O}(N)$}
\vspace{0.2em}

\noindent During PCG, we must map an intermediate matrix $W \in \mathbb{R}^{n\times r}$ to a residual mapping $V \in \mathbb{R}^{n\times r}$ such that $\operatorname{vec}(V) = \mathcal{H} \operatorname{vec}(W)$.

\textbf{Theorem 1 (Implicit Operator Action).} \textit{For any $W \in \mathbb{R}^{n\times r}$, the mapping $V \in \mathbb{R}^{n\times r}$ such that $\operatorname{vec}(V) = \mathcal{H} \operatorname{vec}(W)$ evaluates exactly as:}
\begin{equation}
V = KY + \lambda U, \label{eq:3}
\end{equation}
\textit{where $U = KW \in \mathbb{R}^{n\times r}$ and the $i$-th row of $Y \in \mathbb{R}^{n\times r}$ is uniquely determined by $Y_{i,:} = U_{i,:} E^{(i)}$. This requires strictly $\mathcal{O}(n^2r + nr^2)$ time, completely independent of $q, M$, and $N$.}

\textit{Proof.} By the Kronecker-vector identity $\operatorname{vec}(AXB) = (B^T \otimes A) \operatorname{vec}(X)$, applying the first term evaluates as $(Z \otimes K) \operatorname{vec}(W) = \operatorname{vec}(KWZ^T) = \operatorname{vec}(UZ^T)$. The operator $SS^T \in \mathbb{R}^{N\times N}$ is exactly a sparse diagonal projection mask spanning the unaligned data entries. When applied to $\operatorname{vec}(UZ^T)$, it zeros out all unobserved tensor entries. Let the unvectorized masked matrix be $F \in \mathbb{R}^{n\times M}$. Its elements are identically zero everywhere except at the $q$ observed multi-indices, where they evaluate to:
\begin{equation*}
F_{i_m, j_m} = (UZ^T)_{i_m, j_m} = U_{i_m, :} Z_{j_m, :}^T = U_{i_m, :} z^{(m)}.
\end{equation*}

Next, applying the left operator $(Z \otimes K)^T = (Z^T \otimes K)$ yields $\operatorname{vec}(KFZ)$. Let $Y = FZ \in \mathbb{R}^{n\times r}$. Because $F$ is sparse, the $i$-th row of $Y$ expands strictly over the non-zeros in the $i$-th slice:
\begin{equation*}
Y_{i,:} = \sum_{j=1}^M F_{i,j} Z_{j,:} = \sum_{m \in \Omega_i} F_{i_m, j_m} Z_{j_m, :} = \sum_{m \in \Omega_i} \left( U_{i,:} z^{(m)} \right) (z^{(m)})^T = U_{i,:} \sum_{m \in \Omega_i} z^{(m)} (z^{(m)})^T = U_{i,:} E^{(i)}.
\end{equation*}

Finally, multiplying by $K$ completes the main term yielding $\operatorname{vec}(KY)$. The regularization term simplifies seamlessly: $\lambda(I_r \otimes K) \operatorname{vec}(W) = \operatorname{vec}(\lambda U)$. Summing the components validates $V = KY + \lambda U$. \hfill $\square$

\vspace{0.5em}
\noindent\textbf{3.3. Preconditioner Design: Row-wise Block-Jacobi}
\vspace{0.2em}

\noindent\textbf{Theorem 2 (Exact Block-Diagonal Operator).} \textit{Let $M^{(i)} \in \mathbb{R}^{r\times r}$ denote the exact $i$-th diagonal block of the Hessian operator $\mathcal{H}$ acting linearly on the $i$-th row of $W$ (represented as a column vector). It evaluates exactly as:}
\begin{equation}
M^{(i)} = \sum_{l=1}^n K_{i,l}^2 E^{(l)} + \lambda K_{i,i} I_r. \label{eq:4}
\end{equation}

Since $E^{(l)} \succeq 0$, $K_{i,l}^2 \ge 0$, and RKHS kernels satisfy $K_{i,i} \ge 0$, the resultant blocks are positive semi-definite, and symmetric positive definite for any $\lambda > 0$, allowing for computation of the Cholesky-like factorization $M^{(i)} = L_i L_i^T$. The sum $\sum_l K_{i,l}^2 E^{(l)}$ can be rapidly evaluated in $\mathcal{O}(n^2r^2)$ operations via flattened dense matrix multiplication.

\vspace{1em}
\noindent\textbf{Algorithm 1: PCG for Unaligned RKHS Mode-$k$ Subproblem (Approach 1)}
\vspace{0.5em}
\hrule
\vspace{0.5em}
\begin{algorithmic}[1]
\STATE \textbf{Initialize Setup:} Calculate $Z_\Omega \in \mathbb{R}^{q\times r}$, $\{E^{(i)}\}_{i=1}^n$, and $C = KB \in \mathbb{R}^{n\times r}$.
\STATE \textbf{Preconditioner:} Construct $M^{(i)}$ via flat multiplication, and compute decomposition $L_iL_i^T = M^{(i)}$ for $i = 1 \dots n$.
\STATE $W_0 = \mathbf{0}_{n\times r}$, $R_0 = C$, and $Z_0 \in \mathbb{R}^{n\times r}$.
\FOR{$i = 1 \dots n$}
\STATE Solve $L_iL_i^T(Z_0)_{i,:}^T = (R_0)_{i,:}^T$ \COMMENT{$\triangleright$ Apply Preconditioner Block-wise}
\ENDFOR
\STATE $P_0 = Z_0$
\FOR{$k = 0, 1, 2, \dots$ \textbf{until} $\|R_k\|_F < \tau$}
\STATE $U = KP_k$
\STATE Compute $Y_{i,:} = U_{i,:}E^{(i)}$ for all $i = 1 \dots n$
\STATE $V_k = KY + \lambda U$ \COMMENT{$\triangleright$ Exact fast MVP defined in Theorem 1}
\STATE $\alpha_k = \text{Tr}(R_k^T Z_k) / \text{Tr}(P_k^T V_k)$
\STATE $W_{k+1} = W_k + \alpha_k P_k$
\STATE $R_{k+1} = R_k - \alpha_k V_k$
\FOR{$i = 1 \dots n$}
\STATE Solve $L_iL_i^T(Z_{k+1})_{i,:}^T = (R_{k+1})_{i,:}^T$
\ENDFOR
\STATE $\beta_k = \text{Tr}(R_{k+1}^T Z_{k+1}) / \text{Tr}(R_k^T Z_k)$
\STATE $P_{k+1} = Z_{k+1} + \beta_k P_k$
\ENDFOR
\RETURN $W_{k+1}$
\end{algorithmic}
\vspace{0.5em}
\hrule
\vspace{1em}

\vspace{1em}
\noindent\textbf{\large 4. Approach 2: On-the-fly Sparse Accumulation \& Kernel Preconditioner}
\vspace{0.5em}

\noindent\textbf{4.1. Algebraic Reformulation of the Linear System}
\vspace{0.2em}

\noindent The linear system to solve for $W \in \mathbb{R}^{n\times r}$ is $\mathcal{H}\operatorname{vec}(W) = b$. Using the mixed-product property of Kronecker products, $(A \otimes B)(C \otimes D) = (AC) \otimes (BD)$, we can factor $Z \otimes K = (Z \otimes I_n)(I_r \otimes K)$. Let $P = I_r \otimes K$. Substituting this factorization and setting $H_{\text{red}} = (Z \otimes I_n)^T SS^T (Z \otimes I_n)$, we seamlessly rewrite the exact system as:

\begin{equation}
(P H_{\text{red}} P + \lambda P) \operatorname{vec}(W) = P \operatorname{vec}(B) \label{eq:5}
\end{equation}

\vspace{0.5em}
\noindent\textbf{4.2. Method and Choice of Preconditioner}
\vspace{0.2em}

\noindent We select $P = I_r \otimes K$ as our preconditioner.
Without preconditioning, the system matrix can be ill-conditioned, especially for certain RKHS kernels with decaying eigenvalues. By preconditioning with $P$, the effective operator is normalized to $P^{1/2} H_{\text{red}} P^{1/2} + \lambda I_{nr}$. The condition number becomes mathematically bounded regardless of the kernel spectrum:

\begin{equation}
\kappa \le \frac{\lambda + \lambda_{\max}(H_{\text{red}}) \lambda_{\max}(K)}{\lambda} = 1 + \frac{\lambda_{\max}(H_{\text{red}}) \lambda_{\max}(K)}{\lambda} \label{eq:6}
\end{equation}

\textbf{The ``Inverse-Free'' Tracking Technique:} Standard PCG evaluates $z_k = P^{-1} r_k$ at every iteration, seemingly requiring an intractable $\mathcal{O}(n^3)$ inversion of $K$. However, note that the right-hand side $b = P \operatorname{vec}(B)$ is in the range of $P$, and the operator $P(H_{\text{red}} P + \lambda I)$ always maps vectors into the range of $P$. By induction, the residual $r_k$ intrinsically factors as $r_k = P \hat{r}_k$. Therefore, the preconditioned residual is exactly $z_k = \hat{r}_k$. By analytically updating $\hat{r}_k$ instead of $r_k$, $P^{-1}$ perfectly algebraically cancels out.

\vspace{0.5em}
\noindent\textbf{4.3. Matrix-Free Matrix-Vector Products (Sparse Operations)}
\vspace{0.2em}

\noindent We evaluate $U = \operatorname{unvec}(H_{\text{red}} \operatorname{vec}(V))$ in $\mathcal{O}(qr)$ time by bypassing $N$ and $M$.

\begin{enumerate}
    \item \textbf{Apply $(Z \otimes I_n)$:} Conceptually expands the model to the full unaligned grid.
    \item \textbf{Apply Sparse Selection $S^T$:} Extract the $q$ observed entries. For observation $m$, the predicted value is $(VZ^T)_{i_m, j_m} = \langle V_{i_m, :}, z_m \rangle$. We compute these $q$ scalars $e_m$.
    \item \textbf{Apply Adjoint $S$ and $(Z^T \otimes I_n)$:} The output $U$ sparsely accumulates the entries where the mode-$k$ index matches $i$: $U_{i, :} = \sum_{m: i_m = i} e_m z_m$.
\end{enumerate}

\vspace{1em}
\noindent\textbf{Algorithm 2: Complete Inverse-Free PCG Algorithm (Approach 2)}
\vspace{0.5em}
\hrule
\vspace{0.5em}
\begin{algorithmic}[1]
\STATE \textbf{Setup Phase:}
\STATE Pre-extract $Z_\Omega \in \mathbb{R}^{q\times r}$ using Khatri-Rao rules.
\STATE Initialize RHS state with $q$ known tensor values $y_m$: $B \in \mathbb{R}^{n\times r}$ where $B_{i,:} = \sum_{m: i_m = i} y_m z_m$.
\STATE Initialize: $W_0 = \mathbf{0}$, $Z_0 = B$, $D_0 = B$. Compute $\tilde{Z}_0 = KB$, and let $V_0 = \tilde{Z}_0$.
\FOR{$k = 0, 1, \dots$ \textbf{until} convergence}
\STATE \textbf{Sparse MVP ($\mathcal{O}(qr)$):}
\STATE \quad $e_m = \langle (V_k)_{i_m, :}, z_m \rangle$ for $m = 1 \dots q$
\STATE \quad Initialize $U_k = \mathbf{0}_{n\times r}$. Accumulate $(U_k)_{i_m, :} \mathrel{+}= e_m z_m$
\STATE \textbf{Compute Operator Action ($\mathcal{O}(nr)$):}
\STATE \quad $Q_k = U_k + \lambda D_k$
\STATE \textbf{Step Size ($\mathcal{O}(nr)$):}
\STATE \quad $\alpha_k = \frac{\text{Tr}(Z_k^T \tilde{Z}_k)}{\text{Tr}(V_k^T Q_k)}$
\STATE \textbf{Update Estimates ($\mathcal{O}(nr)$):}
\STATE \quad $W_{k+1} = W_k + \alpha_k D_k$
\STATE \quad $Z_{k+1} = Z_k - \alpha_k Q_k$
\STATE \textbf{Update Conjugation Maps ($\mathcal{O}(n^2r)$):}
\STATE \quad Compute $M_k = KU_k$ \COMMENT{$\triangleright$ Only dense $n \times n$ multiplication required}
\STATE \quad $\tilde{Z}_{k+1} = \tilde{Z}_k - \alpha_k(M_k + \lambda V_k)$
\STATE \textbf{Update Directions ($\mathcal{O}(nr)$):}
\STATE \quad $\beta_k = \frac{\text{Tr}(Z_{k+1}^T \tilde{Z}_{k+1})}{\text{Tr}(Z_k^T \tilde{Z}_k)}$
\STATE \quad $D_{k+1} = Z_{k+1} + \beta_k D_k$
\STATE \quad $V_{k+1} = \tilde{Z}_{k+1} + \beta_k V_k$ \COMMENT{$\triangleright$ Tracks $KD_{k+1}$ rigorously}
\ENDFOR
\STATE \textit{(Note: Every 50 iterations, explicitly refresh $\tilde{Z}_k = KZ_k$ and $V_k = KD_k$ to clear drift).}
\end{algorithmic}
\vspace{0.5em}
\hrule
\end{tcolorbox}

\end{document}